\documentclass{article}

\PassOptionsToPackage{numbers, compress}{natbib}

\usepackage[preprint]{neurips_2025}

\usepackage[utf8]{inputenc} 
\usepackage[T1]{fontenc}    
\usepackage{hyperref}       
\usepackage{url}            
\usepackage{booktabs}       
\usepackage{amsfonts}       
\usepackage{nicefrac}       
\usepackage{microtype}      
\usepackage{xcolor}         
\usepackage[table]{xcolor}

\usepackage{svg}
\usepackage{amsthm}
\usepackage{amsmath}
\usepackage{amssymb}
\usepackage{graphicx}
\usepackage{multirow}
\usepackage{enumitem}
\usepackage{bbm}
\usepackage{wrapfig}
\usepackage{tocloft}
\usepackage[normalem]{ulem}

\newtheorem{theorem}{Theorem}
\newtheorem{lemma}[theorem]{Lemma}
\newtheorem{prop}[theorem]{Proposition}
\newtheorem{cor}{Corollary}[theorem]

\definecolor{darkgreen}{rgb}{0.0, 0.5, 0.0}

\title{Decom-Renorm-Merge: Model Merging on the Right Space Improves Multitasking}

\author{
  Yuatyong Chaichana \\
  Chulalongkorn University \\
  \And
  Thanapat Trachu \\
  Chulalongkorn University \\
  \And
  Peerat Limkonchotiwat \\
  AI Singapore \\
  \AND
  Konpat Preechakul \\
  UC Berkeley \\
  \And
  Tirasan Khandhawit \\
  Mahidol University \\
  \And
  Ekapol Chuangsuwanich\thanks{Correspondence to \texttt{{ekapol.c@chula.ac.th}}.} \\
  Chulalongkorn University \\
}

\begin{document}

\maketitle
\setcounter{footnote}{0}

\begin{abstract}
In the era of large-scale training, model merging has evolved into a tool for creating multitasking models efficiently. It enables the knowledge of models to be fused, without the need for heavy computation as required in traditional multitask learning. Existing merging methods often assume that entries at identical positions in weight matrices serve the same function, enabling straightforward entry-wise comparison and merging. However, this assumption overlooks the complexity of finetuned neural networks, where neurons may develop distinct feature compositions, making direct entry-wise merging problematic. We present \textbf{Decom-Renorm-Merge (DRM)}, a simple yet effective approach that leverages Singular Value Decomposition to decompose and coordinate weight matrices into an aligned joint space, where entry-wise merging becomes possible. We showcase the effectiveness of DRM across various settings ranging from smaller encoder-based such as ViT and DeBERTa, encoder-decoder-based such as T5, and larger decoder-based such as Llama3.1-8B. Our experimental results show that DRM outperforms several state-of-the-art merging techniques across full finetuning and low-rank adaptation settings. Moreover, our analysis reveals renormalization as the crucial component for creating a robust and even joint space for merging, significantly contributing to the method's performance.\footnote{Code and models are available at \url{https://github.com/yophis/decom-renorm-merge}.}

\end{abstract}

\section{Introduction}
Multitask learning \citep{multitasklearning} aims to equip machine learning models with the generalist ability of humans---to acquire and perform multiple tasks using a single model.
This capability is best demonstrated in large language models \citep{t5, llama, gpt3, deepseekv3, gemini}, which can simultaneously handle a wide range of tasks requiring language understanding.
However, acquiring high-quality training data—especially for specialized domains such as medicine, law, or finance—is often constrained by proprietary access, licensing, or privacy concerns. Finding the right task and data mixtures requires costly trial and error \citep{modelmerginganddatamixture}, where even a single training run can be a major investment. These challenges have motivated research into model merging as an alternative.

Rather than training a multitask model from scratch, model merging combines independently finetuned checkpoints into a single multitask-capable model, without requiring additional training. This also permits the reuse of an increasing number of open-source finetuned models. \citet{taskarithmetic} introduces the use of \emph{task vector}---defined as the difference between finetuned and base parameters---as a way to perform various arithmetic operations such as addition for task fusion, and subtraction for task forgetting. Many merging techniques adopt task vectors, since they better capture the finetuning dynamic and update direction. Furthermore, it is found that a large portion of task vector entries can be pruned without heavily impacting task-specific performance \citep{tiesmerging}. This sparsity also enhances merging performance by reducing parameter interferences between models.

Parameter interference is identified in TIES Merging \citep{tiesmerging} as a key obstacle to an effective model merging. They demonstrate that two major interferences are redundant parameter values and sign disagreements at corresponding positions in the weight matrices. \citet{tiesmerging} propose to address the issue by first pruning small entries from the task vectors. Then, in each position, one dominant sign is selected, and all other entries with opposite sign are zeroed out. The pruned entries are ignored during the final averaging step.

These interference mitigation operations, employed in TIES \citep{tiesmerging} and its variants \citep{dare}, rely on a crucial assumption: task vector entries at identical positions serve consistent semantic roles and functions. This assumption allows entry-wise comparison, pruning, dominant sign selection, and the subsequent averaging to be done. However, this fundamental assumption may not hold in practice, as: (a) finetuning can lead to changes in feature arrangement of the activations \citep{neuronequivariance, permutationinvariance, gitrebasin}, and (b) a single neuron often carries not just a single, but a composition of several overlapping features, a phenomenon known as polysemanticity \citep{toymodelofsuperposition, polysemanticityandcapacity, multifactedfeature}. Based on these observations, we formulate our central hypothesis: for entry-wise interference reduction strategies to be valid during model merging, task vectors \emph{must first be decomposed and then coordinated into a shared representation space}. 

Following our hypothesis, we propose \textbf{Decom-Renorm-Merge} (DRM), a method that merges models in a shared representation space. Unlike previous approaches that operate on task vectors (task-specific parameter differences flattened into vector form), DRM preserves the matrix structure of updates—referred to as the \textit{weight delta}—which is essential for decomposing into a shared basis. As shown in Figure~\ref{fig:method}, DRM consists of four key steps:
(1) \textbf{Decompose} a concatenated set of weight delta matrices using SVD to extract a shared basis and project each model into this joint space.
(2) \textbf{Renormalize} each model’s projected basis vectors to unit length; this step is crucial for stable merging.
(3) \textbf{Prune:} Keep only the top magnitude entries in each renormalized basis vector matrix to reduce interference.
(4) \textbf{Merge:} Combine the pruned weights across models using the rest of TIES’ interference mitigation techniques~\cite{tiesmerging}.

\begin{figure}[t!]
    \centering
    \includegraphics[width=1.0\linewidth]{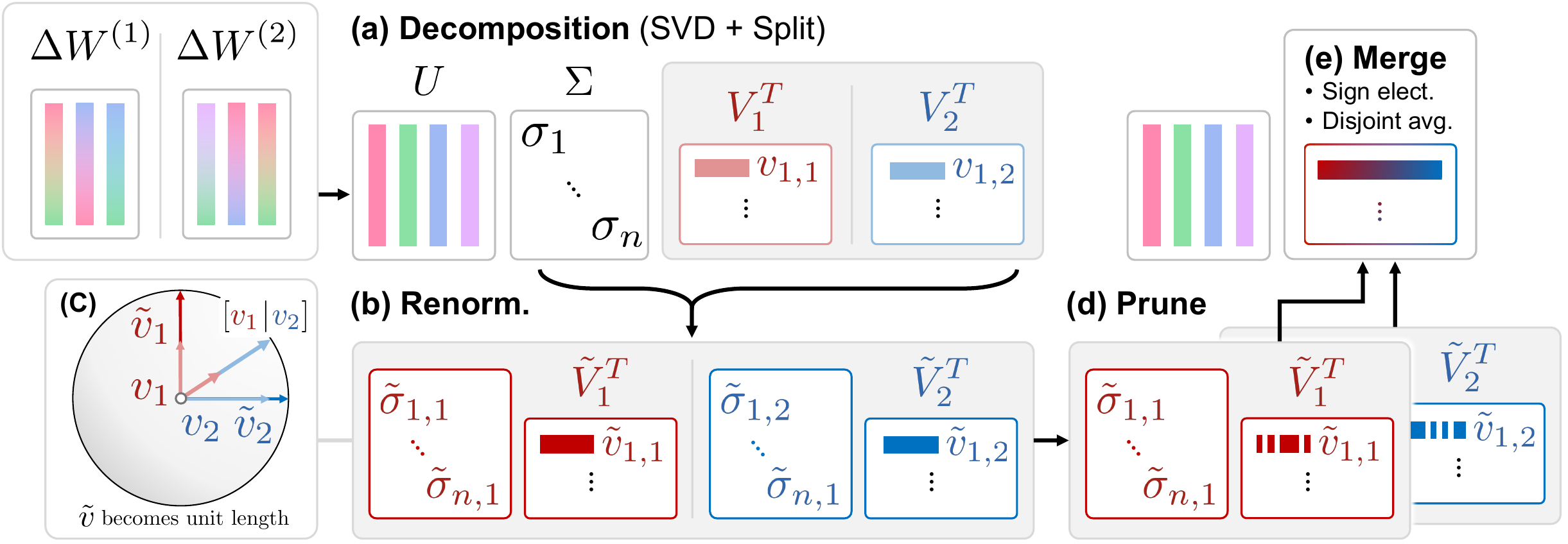}
    \caption{\label{fig:method} \textbf{Decom-Renorm-Merge (DRM)} is a model-merging method for building multitask models. Different models may not share the same weight parameterization. Thus, \textit{merging should occur in a shared decomposed weight space, not the original parameter space}. DRM merges models' \textit{weight deltas} $\Delta W^{(t)}$---the difference of each finetuned model from a shared base model---into a single merged delta. DRM consists of four main steps: \textbf{(a) Decompose:} Concatenate the $\Delta W^{(t)}$ matrices horizontally (or vertically), then apply SVD to decompose them into a shared basis $U$ and individual weights $V_t$. Although the combined $V$ is orthonormal, each individual split $V_t$ is not. \textbf{(b)~Renormalize} each row vector $v_{t,i}$ of $V_t$ to unit length, and scale the corresponding singular value to preserve the magnitude. This can be viewed as compensating for $V_t$’s non-orthonormality as illustrated in \textbf{(c)}. \textbf{(d)~Prune} each renormalized individual singular vector matrix $\tilde{V}_{t}$ by keeping only the entries within top-$k$\% magnitudes. \textbf{(e)~Merge} the pruned singular vector matrices across models using sign election and disjoint averaging.
    }
    \vspace{-1.5em}
\end{figure}

We demonstrate the effectiveness of DRM across various setups, including: (1) Diverse model architectures and sizes. Ranging from smaller encoder-based ViT-B/32, ViT-L/14 \citep{vit, clip}, and DeBERTa-Base \citep{debertav3}; mid-sized, encoder-decoder model T5-Base and T5-Large \citep{t5}; to large decoder-based Llama3.1-8B \citep{llama3}. (2) Comprehensive task categories and modalities, including image recognition and natural language understanding. (3) With and without validation set for merging techniques' hyperparameter tuning. Notably, even with no validation set available, DRM outperforms the strongest baselines by $5.0\%$ and $1.9\%$ in merging ViT-B/32 and ViT-L/14, respectively, $9.3\%$ on DeBERTa-Base-v3, and $1.9\%$ on low-rank (LoRA) adapted Llama3.1-8B, on absolute scales. Additionally, we demonstrate that DRM has better performance retention as the number of tasks being merged increases, compared to other baselines.

Through detailed analysis of the aligned joint merging spaces, we uncover that naively performing interference mitigation and merging operations in the joint space obtained from SVD yields surprisingly minimal improvements. Instead, we identify \emph{renormalization} as the crucial enabling component of our method. The technique involves normalizing the split singular vectors back to unit length after they lose this property due to the splitting. Without this key ingredient, the singular vector bases demonstrate significant instability, leading to an erratic pruning outcome where some top principle components are entirely zeroed out, causing severe performance degradation.

In particular, our work makes the following main contributions:
\begin{itemize}[leftmargin=*]    
    \item We propose \textbf{Decom-Renorm-Merge (DRM)}, a novel model merging technique that constructs a shared representation space for more effective model merging. 
    \item We demonstrate \textbf{state-of-the-art performance} across both vision and language domains, including CLIP’s vision encoders ViT-B/32 and ViT-L/14, DeBERTa-Large-V3, and LoRA-finetuned LLaMA 3.1-8B, showcasing the effectiveness of our method in both full and parameter-efficient settings.
    \item We identify the critical role of renormalization (Section~\ref{sec:understanding the decompose-renormalized joint space}) in enabling stable and robust merging within the shared decomposed weight space, through a series of empirical and theoretical analyses.
\end{itemize}

\section{Related Works}
\label{sec:relatedworks}

\paragraph{Model merging and weight combination.}
Early research on model merging starts by means of averaging together different finetuned weights derived from identical model architectures and tasks, with a goal to improve the final performance of that specific task \citep{modelsoup, stochasticweightaveraging, parallelstochasticweightaveraging}, and generalization ability \citep{ensembleofaverages, swad, modelratatouille, diverseweightaveraging}. Candidate weights for averaging can be obtained either from different optimization steps or entirely different training runs. \citet{fusingforfinetuning} show that the averaged weights can be used as a better initialization point for new task finetuning. Merging was later extended to combining weights of different tasks, creating a single multitasking model. Pioneered by \citet{fishermerging}, rather than simply averaging the parameters, the previous work weights the averaging using Fisher information matrices, effectively capturing the importance of each parameter. But the Fisher information has to be calculated from the network's gradient. Regression Mean \citep{regmean} frames merging problem as linear regression of each individual linear layer in the model. While a closed-form solution exists, it requires additional activation statistics from the models. Task Arithmetic \citep{taskarithmetic} proposes task vector, which is attained by subtracting the base out from finetuned weight. Inspired by task analogy and regular vector arithmetic, they showcase effectiveness on applying arithmetic operations to the task vectors, such as addition for task merging, and subtraction for task forgetting. \citep{tangenttaskarithmetic} introduce a novel linearized finetuning on the tangent plane, in order to facilitate downstream task vector arithmetic. TIES Merging \citep{tiesmerging} adheres to the use of task vector. They demonstrate that a combination of parameter pruning, and dominant sign election, can resolve parameter interference and significantly enhance merging performance. DARE \citep{dare} displays excessive redundancy within task vectors. They propose random parameter pruning and rescaling as remedies, paving the way for other merging methods to be applied effectively afterwards.

\paragraph{Mode connectivity and weight permutation.}
Even with the complex and highly non-convex nature of the loss surfaces of deep neural networks, it was found that there exists paths, possibly linear, of near-constant low loss between independently optimized large models, a phenomenon called mode connectivity \citep{topologyofneuralnetworkoptimization, modeconnectivity, modeconnectivity2, linearmodeconnectivityandlottery}. Based on these findings, a body of works \citep{permutationinvariance, modeconnectivityandneuronalignment} argue the probable existence of weight permutation symmetry in neural network. Namely, two differently trained models can be permuted and aligned into a same convex basin, in which there is a linearly interconnecting path. This subsequently allows the permuted networks to be merged simply via weight averaging \citep{optimaltransportpermutation, graphmatchingpermutation, gitrebasin, repairpermutation, sinkhornpermutation, zipit}.

\section{Motivation: Merging models on the right space}

\paragraph{Problem settings.}
Our goal is to merge several finetuned models, each derived from the same pretrained base, into a single multitask model while maintaining maximum performance across all merged tasks, \underline{without} involving any forward-backward calculation or architectural change.

Concretely, given a pretrained model with $L$ linear layers, let $ W_l^\text{PT} \in \mathbb{R}^{d_{out}^{(l)} \times d_{in}^{(l)}} $ denote the weight matrix of the $l$-th layer. Here, $d_{in}^{(l)}$ and $d_{out}^{(l)}$ represent the input and output dimensions of layer $l$, respectively. We define $ \mathbb{W}^\text{PT} = \{ W_l^\text{PT} \mid l \in [1, L] \} $ as the set of weights of the pretrained model. We can perform either full or parameter-efficient finetuning on the base model towards $ N $ downstream tasks, obtaining task-specific models $ \{ \mathbb{W}^{(1)},\mathbb{W}^{(2)}, \dots, \mathbb{W}^{(N)} \} $. 

The \textbf{objective} is to combine the finetuned models into a single multitask-capable model $ \mathbb{W}^\text{M} $ that performs well on all of the $ N $ tasks on average. In our settings, we consider merging the \emph{weight deltas}, which represent the parameter difference from the base model in matrix form; defined as $ \Delta W^{(t)}_l = W^{(t)}_l - W^\text{PT}_l$, for each task $t$ and layer $l$, instead of merging the model weights directly.

\subsection{Different Finetuned Models may not Share the Same Weight Parameterization}
Finetuned models typically remain close to the base pretrained model in parameter space—resulting in small \textit{weight deltas} \citep{finetuningfineness}—this might suggest that merging such models, which seemingly share similar weight configurations and semantics, should be straightforward. However, our argument is that \emph{coordinating the weights before merging is essential}. This is because, despite the small magnitude and proximity of the weight deltas, the arrangement of features within them can still diverge across tasks for the following reasons:

\paragraph{Finetuning can lead to change in contextualization.}
Weights in the hidden layers are contextualized by the adjacent layers, in a way that their feature configurations and orderings have to be in accordance. But as parameter updates from finetuning propagated into the network, it leads to changes in activations, allowing alteration to the contextualization and thus feature arrangement.

\paragraph{The optimization path is non-unique in parameter space.}
Although, under the infinite width assumption, neural networks follow a deterministic trajectory in \textit{functional space} during training \citep{NTK}, this does not guarantee a unique solution in \textit{parameter space}. 
In fact, different model parameterizations may lead to functionally equivalent models. These weight symmetries are supported by empirical findings in weight permutation studies \citep{neuronequivariance, wideanddeep, modeconnectivityandneuronalignment, permutationinvariance, gitrebasin, repairpermutation, zipit}, which suggest that even under the same pretrained feature space, functionally similar features can be manipulated into different arrangements throughout the course of finetuning.

\subsection{Neuron Polysemanticity in Neural Networks}
\label{sec:neuronpolysemanticity}
Artificial neurons are often assumed to fire when recognizing a single specific feature it learned. But recent findings demonstrate that these neurons are often polysemantic \citep{introtocircuits, circuitsininceptionv1, toymodelofsuperposition, polysemanticityandcapacity, multifactedfeature}, responding to multiple seemingly unrelated concepts \citep{cnnfeaturevisualization}. With a single neuron carrying a composition of features, \emph{weight delta decomposition is required} to enable further interpretation and comparison of their underlying representations \citep{representationlearning, betavae, infogan, disentanglebyfactorization, solu}. In the rest of this work, we use the term \emph{feature} and \emph{representation} interchangeably.

\section{Proposed Method: Decom-Renorm-Merge}
\label{sec:proposedmethod}

\subsection{Hypothesis}
\label{sec:hypothesis}

Building on the above observations, we know that weight delta entries at identical positions do not necessarily encode consistent semantics or functionalities. This invalidates direct entry-wise interference reduction of the weight delta matrices. Therefore, we make the following hypothesis: weight deltas of differently adapted models \emph{must be decomposed and then coordinated into an identical representation space}, before proper interference reduced merging can be carried out.

\subsection{Algorithm}
\label{sec:algorithm}
We propose \textbf{Decom-Renorm-Merge (DRM)} which comprises (1) shared decomposition via SVD, for aligning different weight deltas into a joint representation space, (2) per-task renormalization for stabilizing each individual basis and enabling downstream interference reductions, (3) TIES' operations applied in the Decompose-renormalized joint space, including pruning, sign election, and disjoint averaging. See Appendix \ref{sec: Analysis of Merging Models through Different Decompositions} for analysis on different choice of decompositions.

Concretely, our method consists of the following steps:

\textbf{1) Joint Decomposition:} For each linear layer $l$, we concatenate the weight deltas $\{\Delta W_l^{(t)}\}_{t=1}^N$ horizontally across $N$ tasks, and apply SVD,
\begin{align}
    \Delta W_l^{\text{stack}} &= [ \Delta W_l^{(1)} \; \Delta W_l^{(2)} \; \cdots \; \Delta W_l^{(N)} ] \\
    &= U \Sigma V^T
\end{align}
where $U \in \mathbb{R}^{d_{out}^{(l)} \times d_{out}^{(l)}}$ is the shared left singular vector matrix (representing joint output feature basis), $\Sigma \in \mathbb{R}^{d_{out}^{(l)} \times Nd_{in}^{(l)}}$ contains singular values (capturing each basis direction's shared importance across task), and $V^T \in \mathbb{R}^{Nd_{in}^{(l)} \times Nd_{in}^{(l)}}$ is the right singular vector matrix.

\textbf{2) Partition:} We partition the full right singular vector matrix $ V^T $ into task-specific blocks,
    \begin{align}
        \Delta W_l^{\text{stack}} &= U \Sigma [ V_1^T \; V_2^T \; \cdots \; V_N^T ]
    \end{align}
    where each $V^T_t \in \mathbb{R}^{Nd_{in}^{(l)} \times d_{in}^{(l)}}$ corresponds to task $t$. Note that each task's original weight delta can be reconstructed by,
    \begin{align}
        \Delta W_l^{(t)} = U \Sigma V^T_t
    \end{align}
    
\textbf{3) Renormalization.} The rows of $V^T_t$ are generally not of unit lengths due to the partitioning. For each row $ {v}_{t,i} $, we renormalize the vector,
    \begin{align}
        \tilde{v}_{t,i} &= \frac{v_{t,i}}{||v_{t,i}||_2}
    \end{align}
    Collectively, this yields a renormalized matrix $\tilde{V}_t^T$. The original row norms are redistributed into the singular values,
    \begin{align}
        \Sigma_{t} = \Sigma \widehat{V}_t
    \end{align}
    where $\widehat{V}_t$ is a diagonal matrix with entries $\|v_{t,i}\|_2$ on the diagonal, and $\Sigma_{t}$ is now the singular value matrix of the renormalized task $t$ specifically.

\textbf{4) Pruning:} For each task $t$, we keep only the top-$k\%$ largest-magnitude entries in $\tilde{V}^T_t$, and zero out the others. It is vital that this step be done after renormalization to ensure stability, pruning on unnormalized magnitudes could distort the basis vectors. We default this to top-$20\%$ \citep{tiesmerging}.

\textbf{5) Sign Election:} To resolve inter-task sign conflicts, we compute a sign mask, $S= \operatorname{sgn}(\sum_{t=1}^N \Sigma_t \tilde{V}_t^T) $, selecting a magnitude dominant sign for each position. Then, we zero out entries with signs opposite to the elected mask. This procedure is done in scale-aware fashion by multiplying the $\Sigma_{t}$ in.

\textbf{6) Disjoint Averaging:} We compute the disjoint average of the processed $ \Sigma_t \tilde{V}_t^T $, considering only non-zero entries, 
\begin{align}
    (\Sigma\tilde{V}^T)^{\text{M}} = \Gamma \odot \sum_{t=1}^N \lambda_t \Sigma_t \tilde{V}_t^T
\end{align}
Here, $\lambda_t$ is a task-specific scalar weighting coefficient, with a default $\lambda = 1.0$ for all tasks. $\Gamma $ is an entry-wise averaging reciprocal mask that ensures averaging procedure is performed only over non-zero entries. We define $\Gamma \in \mathbb{R}^{m \times n}$ with entry at row $i$, column $j$ as the following reciprocal:
\begin{align}
    \label{eq:disjoint_mask}
    \Gamma_{i,j} = \frac{1}{\sum_{t=1}^{N} \mathbbm{1} (\Delta W_{t,i,j} \ne 0 )}
\end{align}

\textbf{7) Reconstruction:} We transform the merged representation back into the original parameter space,
\begin{align}
    \Delta W_l^\text{M} &= U(\Sigma \tilde{V}^T)^\text{M}
\end{align}
and recover the full merged weight,
\begin{align}
    W_l^\text{M} = W_l^\text{PT} + \Delta W_l^\text{M}
\end{align}

For other parameters such as bias terms and affine weights, weighted averaging is applied using the weights $\lambda_t$. Ultimately, the same algorithm can also be carried out vertically through $U_t$, as for aligning into joint row basis instead. We denote the horizontal stacking version of our algorithm \emph{DRM-H}, and the vertical counterpart \emph{DRM-V}. We demonstrate their differences in the subsequent sections.

\section{Experiments}
\label{sec:experiments}
In this section, we test our proposed merging method experimentally. We directly evaluate the merging performance on different model architectures, tasks, and finetuning approaches. We then test the method's sensitivity and ablate the significance of our algorithmic component.

\subsection{Baselines}
For comparison, we replicate and evaluate our methods against the following four methods:
\begin{enumerate}[leftmargin=*]
    \item \textbf{Simple Averaging} compute the average of every model's weights, $ W_l^\text{M} = \frac{1}{N} \sum^N_{t=1} W_l^{(t)} $.
    
    \item \textbf{Task Arithmetic} \citep{taskarithmetic} performs a weighted averaging on the weight deltas, $ W_l^\text{M} = W_l^\text{PT} + \sum^N_{t=1} \lambda_i \Delta W_l^{(t)} $. We follow \citep{tiesmerging} and set the default to a shared $ \lambda = 0.4$.
    
    \item \textbf{TIES Merging} \citep{tiesmerging} first prunes the low magnitude weights. Then, reduces sign conflicts by electing a dominant magnitude sign for each parameter location, and keep only parameters with the same sign as the elected. Afterwards, only the kept parameters are averaged. As in the original paper, we set the default to retaining top-$20\%$, and a shared $ \lambda $ of 1.0.
    
    \item \textbf{DARE-TIES} \citep{dare} extends TIES Merging, by replacing magnitude pruning with a random drop of rate $p$ and a rescale with $\frac{1}{1-p}$. We again set the default to retaining top-20\%, and a shared $ \lambda $ of 1.0.
    
\end{enumerate}

\subsection{Results}

\begin{table}[t]
    \caption{Merging results of vision and language models, with and with no validation set available. We report all results in classification accuracy. \textbf{Bold} represents the best score, and \underline{underline} represents the second best.}
    \centering
    \begin{tabular}{lcccccc}
        \toprule
        && \multicolumn{2}{c}{\textbf{Vision Models}} & \multicolumn{3}{c}{\textbf{Language Models}} \\
        \cmidrule(lr){3-4}
        \cmidrule(lr){5-7}
        \textbf{Methods} & \textbf{Validation} & \textbf{ViT-B/32} & \textbf{ViT-L/14} & \textbf{DeB-Base} & \textbf{T5-Base} & \textbf{T5-Large}  \\
        \midrule
        Finetune & & 90.8 & 94.3 & 83.4 & 85.1 & 89.3 \\
        \midrule
        Simple Averaging & $\times$ & 66.1 & 79.6 & \underline{57.6} & 63.3 & 73.2 \\
        Task Arithmetic & $\times$ & 60.6 & 83.4 & 55.4 & \textbf{76.4} & \textbf{80.1} \\
        TIES Merging & $\times$ & 72.5 & 86.1 & 48.7 & 71.8 & 72.3 \\
        Dare-TIES & $\times$ & 71.1 & 85.9 & 47.3 & 72.3 & 72.3 \\
        \rowcolor[HTML]{E3ECFA}
        DRM-H & $\times$ & \textbf{77.5} & \textbf{88.0} & \textbf{66.9} & \underline{73.6} & \underline{74.5} \\
        \rowcolor[HTML]{F0E3FA}
        DRM-V & $\times$ & \underline{72.8} & \underline{86.3} & 54.1 & 72.8 & 73.3 \\
        \midrule
        Task Arithmetic & \checkmark & 70.3 & 84.6 & 58.1 & \textbf{76.4} & \textbf{80.1}  \\
        TIES Merging & \checkmark & 74.3 & \underline{86.9} & 57.2 & 73.7 & \underline{79.5} \\
        Dare-TIES & \checkmark & 74.4 & 86.8 & 57.8 & 74.6 & \underline{79.5} \\
        \rowcolor[HTML]{E3ECFA}
        DRM-H & \checkmark & \textbf{78.3} & \textbf{89.3} & \textbf{67.9} & 74.4 & 75.1 \\
        \rowcolor[HTML]{F0E3FA}
        DRM-V & \checkmark & \underline{74.6} & 86.4 & \underline{61.1} & \underline{75.1} & 77.4 \\
        \bottomrule
    \end{tabular}
    \label{tab:merging_results}
\end{table}

\begin{table}[t]
    \caption{Merging results of low-rank fine-tuned Llama 3.1-8B's across five different tasks with and without validation set scenarios. We report all results in classification accuracy. \textbf{Bold} represents the best score, and \underline{underline} represents the second best. Average performance of finetuned Llama 3.1-8B across every task is 90.5.}
    \setlength{\tabcolsep}{3pt}
    \centering
    \begin{tabular}{lcccccccc}
         \toprule
         && \multicolumn{6}{c}{\textbf{Methods}} \\
         \cmidrule(lr){3-8}
         \textbf{Models} & \textbf{Validation} & Averaging & Task Arithmetic & TIES & Dare-TIES & DRM-H & DRM-V \\
         \midrule
         \multirow{2}{*}{Llama 3.1-8B} & $\times$& \multirow{2}{*}{61.5} & 45.9 & 67.1 & 67.6 & \cellcolor[HTML]{E3ECFA}\textbf{69.5} & \cellcolor[HTML]{F0E3FA}\underline{69.1} \\
         & \checkmark && 66.4 & 69.4 & 68.8 & \cellcolor[HTML]{E3ECFA}\underline{73.3} & \cellcolor[HTML]{F0E3FA}\textbf{74.3} \\
         
         \bottomrule
    \end{tabular}
    \setlength{\tabcolsep}{6pt}
    \label{tab:llm_results}
\end{table}

\paragraph{Merging vision models.} We start with merging of image classification models, following \citep{taskarithmetic}. We use ViT-B/32 and ViT-L/14 \citep{vit} visual encoders of CLIP \citep{clip}, finetune them to eight datasets separately: Cars \citep{stanfordcars}, DTD \citep{dtd}, EuroSAT \citep{eurosat}, GTSRB \citep{gtsrb}, MNIST \citep{mnist}, RESISC45 \citep{resisc45}, SUN397 \citep{sun397}, and SVHN \citep{svhn}. We report the results in Table \ref{tab:merging_results}. Using DRM-H to merge fully finetuned ViT-B/32 and ViT-L/14 yields improvements of $5.0\%$ and $1.9\%$ respectively over the strongest baselines. Tuning on validation sets yields improvements of $3.9\%$ and $2.4\%$ over strongest baselines. Remarkably, DRM-H without hyperparameter tuning outperforms baselines with tuning. For performance breakdown of each task, refer to Appendix \ref{sec:performance breakdown}.

\paragraph{Merging language models.}
For merging of language models, we adopt a similar setting to \citep{regmean} and experiment on encoder-only models: DeBERTa-Large-V3 \citep{deberta}, and encoder-decoder based: T5-Base, and T5-Large \citep{t5}. Following \citep{tiesmerging} closely, we finetune the models to six NLU tasks: question answering (QASC \citep{qasc}, WikiQA \citep{wikiqa}, and QuaRTz \citep{quartz}), paraphrase identification (PAWS \citep{paws}), story completion (Story Cloze \citep{storycloze}), and coreference resolution (Winogrande \citep{winogrande}). Table \ref{tab:merging_results} shows that DRM-H improves the performance by $9.3\%$ in merging DeBERTa-Base-V3. Tuning on validation set improves the margin to $9.8\%$ over the strongest baseline.

\paragraph{Merging LoRA-adapted large language models.}
For merging of large language models, we adapt Llama3.1-8B \citep{llama} to five GLUE tasks individually: MNLI \citep{mnli}, QNLI \citep{glue}, RTE \citep{glue}, COLA \citep{cola}, and SST2 \citep{sst2} For this setup, we perform low-rank adaptation ($r=16$) instead of full finetuning as done previously. Table \ref{tab:llm_results} shows that without validation data (no hyperparameter tuning), DRM-H outperforms other competitors by more than $1.9\%$. In the presence of hyperparameter tuning, DRM-V outperforms the strongest baseline by a margin of $3.9\%$. Notice that, unlike in the previous winning scenarios, where DRM-V is usually subpar or relatively close to DRM-H. Reversed here, DRM-V with hyperparameter tuning surpasses DRM-H by a margin of $1.0\%$. We provide an analysis on this phenomenon in Appendix \ref{sec: connection between v and h}.

\begin{wraptable}{h}{0.5\textwidth}
\vspace{-1.5em}
\setlength\tabcolsep{2pt}
\caption{Comparison of merging results with and without renormalization on DRM-H. We observe that renormalization improves performance of the merged models by a large margin, most prominently in DeBERTa-Base with a $8.8\%$ difference.}
\centering
\vspace{0.7em}
\begin{tabular}{c@{\hskip 4pt}c@{\hskip 5pt}c@{\hskip 5pt}c@{\hskip 5pt}c}
    \toprule
    Renorm & ViT-B/32 & T5-B & DeB-B & Llama-8B \\
    \midrule
    \checkmark & 77.5 & 73.6 & 66.9 & 69.5 \\
    $\times$ & 73.5 & 68.6 & 58.1 & 62.7 \\
    \bottomrule
\end{tabular}
\label{tab:with and without renormalization}
\end{wraptable}

\subsection{Additional Results and Analysis}

\label{sec:additional results}
\paragraph{Merging different number of models.}
In this study, we experiment with varying the number of tasks being merged as done in \citep{taskarithmetic, tiesmerging}. We compare our DRM-H with Simple Averaging, TIES Merging, and DARE-TIES on two models: ViT-B/32 and T5-Base, Each point represents result on a single subset of the tasks, and the solid line represents the average performance for each subset size. Note that we only sample at most 10 distinct combinations for each subset size. From Figure \ref{fig:merging different number of tasks}, we see that at 2 or 3 tasks, performance of most merging techniques stay close together. But as the number of tasks grows, DRM-H maintains highest merging performances, suggesting that DRM-H scales better to a large number of tasks.

\paragraph{Ablation of renormalization.}\label{sec:ablate_renorm}
This experiment investigates the performance impact of renormalization in our merging algorithm. We ablate performing merging, with and without applying renormalization to the joint representation space. Table \ref{tab:with and without renormalization} demonstrates that renormalization provides enormous gain to the merging performance. The improvement ranges from $4.0\%$ in ViT-B/32, $5.0\%$ in T5-Base, $8.8\%$ in DeBERTa-Base, and $6.8\%$ in Llama3.1-8B.

\begin{figure}
    \centering
    \includegraphics[width=0.9\linewidth]{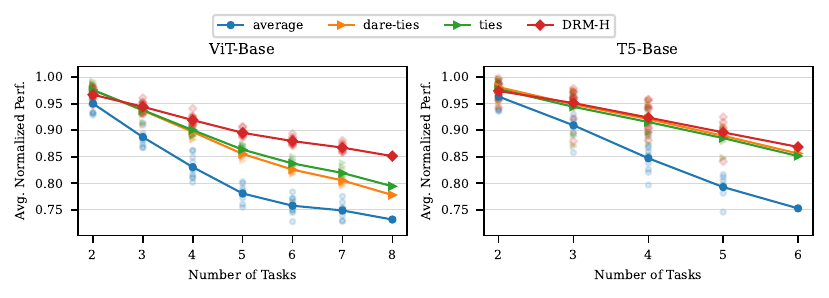}
    \vspace{-1em}
    \caption{Results obtained when merging different number of tasks. DRM-H maintains better performance as the number of merged tasks increased.}
    \label{fig:merging different number of tasks}
    \vspace{-1em}
\end{figure}

\section{Understanding the Decompose-Renormalized Joint Space}
\label{sec:understanding the decompose-renormalized joint space}
Seeing the performance of the DRM, we hope to understand more about our method, particularly the intrinsic nature of the Decompose-renormalized joint space used for the merging. 

We perform a series of studies to answer the following questions: 1) Why is it crucial that magnitude pruning be done on renormalized singular vector matrices? Shouldn't pruning on unrenormalized basis vector matrices be even better? Since they better capture the real scale of each basis entry. 2) How does reducing interference on the Decompose-renormalized space reflect interference on the original parameter space? Does reducing interference in one space directly influence the other? 3) How large is the difference between individual task and the joint decomposed weight delta?

\begin{figure}
    \centering
    \includegraphics[width=1.0\linewidth]{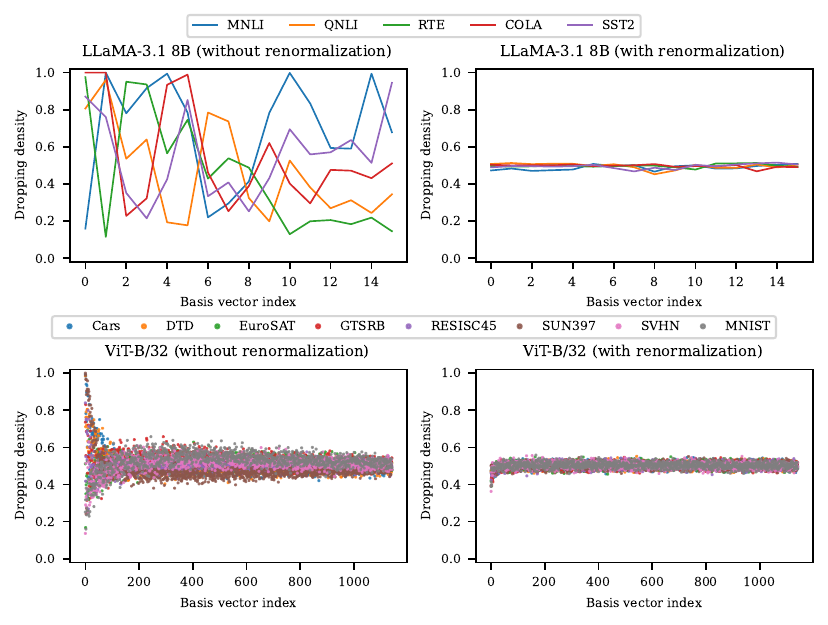}
    \vspace{-1em}
    \caption{
    Percentages of entries being dropped from each row basis vector $v_{t,i}$ during the pruning of entire matrices: \textbf{(Top)} the middle LoRA layer (16th) of Llama-3.1 8B, with and without renormalization. \textbf{(Bottom)} weight deltas of the middle layer (6th) of ViT-B/32, with and without renormalization. Here, we prune $50\%$ of each matrix. The dropped percentages fluctuate severely without renormalization, with some vectors almost entirely zeroed out.
    }
    \label{fig:llama dropping density under renormalization}
    \vspace{-1em}
\end{figure}

\subsection{Why is it Crucial that Weight Pruning be Done on the Renormalized Singular Vector Matrices?}
After partitioning the right singular vector matrix $V^T$ of the concatenated SVD, the row basis vectors lose the orthonormal property. We renormalize each row back to unit norm\footnote{Unlike the unit length property, orthogonality of the split matrices, is not easily reattainable. We observe that enforcing orthogonality through Gram-Schmidt process harshly degrades the performance.} before applying interference reduction techniques. As shown in Section \ref{sec:additional results} that renormalization is central to DRM in terms of performance, here we investigate the effect of renormalization on the pruning behavior, from both empirical and theoretical perspectives.

\paragraph{Empirical investigation.}
Figure~\ref{fig:llama dropping density under renormalization} displays the percentages of entries being dropped from each basis vector. We see that without renormalization, the dropping densities fluctuate severely, with some vectors have almost all of their entries zeroed out. Removing top principle components like this would quickly deteriorate downstream performance. By adding in renormalization, the dropping densities are stabilized evenly across all vectors, bounded to around the configured prune rate of 50\%, no single vector is wholly pruned.

\paragraph{Theoretical investigation.}
The following proposition presents a theoretical explanation for the empirically observed pruning fluctuation when renormalization is \textit{not} applied.
\begin{prop}[Shared Norm Budget of Partitioned Basis Vectors]
\label{prop:pruning_instability}
Let $V^T$ from a horizontally concatenated SVD be partitioned into task-specific blocks, e.g. $V^T = [V_A^T \;\; V_B^T]$ for two tasks. For any given row index $i$, the squared norms of the corresponding row vectors $v_{A,i}^T$ and $v_{B,i}^T$ must satisfy $\|v_{A,i}^T\|_2^2 + \|v_{B,i}^T\|_2^2 = 1$. This creates a "\emph{unit budget}" shared between tasks for each basis direction. Consequently, a larger norm for one task's vector necessitates a smaller norm for the others.
\end{prop}
The proof and formal statement are deferred to Appendix~\ref{sec:Biased Magnitude-based Pruning of Partitioned Singular Basis}. Informally, this unit budget biases magnitude-based pruning to favor preserving entries from the vector with the larger norm, as in the absence of outliers, its entries will have a larger magnitude on average. This can lead to disproportionate pruning of basis vectors associated with tasks that have a smaller share of the norm budget, regardless of the intrinsic importance of their features. Renormalization corrects this pruning bias by equalizing norms of the partitioned basis vectors.

\begin{wraptable}{h}{0.5\textwidth}
    \vspace{-1.2em}
    \setlength\tabcolsep{2pt}
    \caption{Average sign agreement after pruning different spaces. We compare between pruning in the original parameter space (labeled \emph{Original Weight}), versus pruning in the joint singular space. For latter, sign agreement is measured in both the joint singular space itself (labeled \emph{DS}), and after projecting back to the original space (labeled \emph{OS}). We find that reducing interference on the joint space \textit{does not} correspond to equally high sign agreement in the original parameter space. Note that the minimum sign agreement is $50\%$, when there are an equal number of positive and negative entries. }
    \centering
    \vspace{0.3em}
    \begin{tabular}{c@{\hskip 4pt}c@{\hskip 5pt}c@{\hskip 5pt}c}
        \toprule
         & Space  & ViT-B/32 & Llama-8B \\
        \midrule
        Original Weight & - & 86.6 & 91.0 \\
        \midrule
        Horizontal  & DS & 84.4 & 89.7 \\
        & OS & 71.0 & 70.6\\
        \midrule
        Vertical & DS & 86.1 & 89.9 \\
        & OS & 69.5 & 70.2 \\
        \bottomrule
    \end{tabular}
    \vspace{-1.5em}
    \label{tab:sign conflict}
\end{wraptable}

\subsection{How Does Pruning on Decompose-renormalized Space Reflect Interference in Original Parameter Space?}
\label{sec:How Does Pruning on Decompose-renormalized Space Reflect Interference in Original Parameter Space}

In TIES Merging \citep{tiesmerging}, the authors demonstrate that two major interference sources are redundant parameter values, and sign disagreements at corresponding positions in the weight delta matrices. They hypothesized that mitigating these interferences would enhance merging performance.

However, our analysis in Table \ref{tab:sign conflict} reveals that although pruning in the Decompose-renormalized space effectively improves merging performance, it does not result in equally large sign agreement when projected back to the original parameter space. Specifically, we first project the weight deltas into the Decompose-renormalized space, apply pruning to the basis entries. We then transform the results back to the original space and measure sign agreement, but we found the agreement on the original space to be considerably less. 

Nevertheless, DRM still achieves mostly superior merging performance compared to methods reducing interference directly in the original parameter space. This suggests that interference reduction techniques lead to the most effective merging when applied in the Decompose-renormalized space, rather than directly on individual parameter entries. See Appendix \ref{sec:elaboration of Sign Agreement Analysis} for more detailed results and breakdown on different agreement levels.

\subsection{How Large is the Difference between Individual Task and the Joint Decomposed Weight Delta?}
\label{sec:How Large is the Difference between Individual and the Concatenated Weight Delta}

We theoretically evaluate the difference between weight delta of each individual task, and the joint decomposed weight delta used in DRM-H. We express the difference through their corresponding singular values (i.e. $i$-th singular value of a single weight delta matrix, and $i$-th singular value of the concatenated matrix), and establish the following bound:
\begin{prop}[Bounded Difference of Weight Delta Concatenation] \label{prop:Bounded difference of weight delta concatenation}
    Suppose that $t$ is an arbitrary task index, and $M \in \mathbb{R}^{m\times kn}$ is a concatenation of $k$ copies of $\Delta W_t$, where $k$ is the total number of tasks and $r$ is the rank of $M$. Let $\widetilde{M} = [\Delta W_1 \; \dots \; \Delta W_k]$ be the concatenation of the collection $\{\Delta W_j\}_{j=1}^k$. We define the perturbation matrix
    \begin{align*}
        E = \widetilde{M} - M = [E_1 \; \dots E_k], \quad \text{with $E_j = \Delta W_j - \Delta W_t$ \quad for each $j\in[1,k]$\,.}
    \end{align*}
    Then, the difference between $i$-th singular value of the concatenated matrix $\widetilde{M}$ and of the individual $\Delta W_t$, for all $i\in[1, r]$ and $t\in[1,k]$, is bounded by
    \[
    |\sigma_i(\widetilde{M}) - \sqrt{k}\sigma_i( \Delta W_t )| \leq \frac{1}{\sqrt{k} \sigma_i(\Delta W_t)} \sum_{j=1}^k \left( 2 \|\Delta W_t\|_2 \|E_{j}\|_2 + \|E_{j}\|_2^2 \right).
    \]
\end{prop}
This result quantifies the bounded behavior of change to each individual task's weight delta, such that the alteration induced by horizontal concatenation of weight deltas will never exceed a specific amount described by the bound. It is evident that larger singular value $\sigma_i( \Delta W_t )$ will exhibit a tighter bound due to the reciprocal term $\frac{1}{\sigma_i( \Delta W_t )}$ on the right hand side. The proof and additional discussion are provided in Appendix~\ref{sec:Bound of Difference Between Individual Task and the Concatenated Matrix}.

\paragraph{Intuitive interpretation.}

When comparing jointly decomposed weight deltas of different models during interference reduction, it is essential to preserve the top principle features of each model to ensure a fair comparison across tasks; otherwise, certain tasks could be degraded unknowingly. Our derived bound addresses this by indicating a bounded degree of change, especially for the important feature bases characterized by high singular values. It assures the preservation of salient information from each model, facilitating interference reduction and merging in the subsequent steps.

\section{Conclusion}
\label{sec:conclusion}
In this work, we addressed a fundamental limitation in existing interference reduction-based model merging techniques that assume consistent feature composition and arrangement across independently finetuned models. To overcome the challenge, we introduced Decom-Renorm-Merge (DRM), an effective approach that combines singular value decomposition and renormalization for establishing a shared representation space, enabling entry-wise interference reduction techniques and merging to be properly applied within. Our experiments demonstrated that DRM outperformed several state-of-the-art merging techniques across diverse architectures and task modalities. Further analysis provided insights into the crucial role of joint space renormalization, and the necessity of performing interference reduction within the Decompose-renormalized space rather than directly on the parameters. These results further our understanding of effective model merging and provide a means for more efficient knowledge fusion in neural networks.

\section{Acknowledgements}
We are grateful to Artit Suwanbandit and Patawee Prakrankamanant for inspiration and insightful discussions early into this work. Thanks also to Chompakorn Chaksangchaichot for helpful comments that improved the presentation of this paper.

\bibliography{references}


\clearpage

\appendix

\section{Extended Related Works}

\paragraph{Superposition and polysemanticity in neural networks}
Multiple works cue that neural networks encode more features than available dimensions \citep{circuitsininceptionv1}. \citet{cnnfeaturevisualization} and \citet{multifactedfeature} study individual neurons in CNN models and find that neurons are polysemantic---each capable of detecting and responding to several different input patterns. \citet{polysemywordembedding} identifies that word embeddings encode different senses of polysemous words in a linear superposition fashion. \citet{introtocircuits} indicates that, instead of assigning each feature to one dedicated neuron, spreading a single feature across multiple neurons allows deep learning models to pack more features into the limited number of neurons they have. \citet{toymodelofsuperposition} employs toy models to explore the superposition idea. They find that as input feature sparsity increases, models use superposition to represent features beyond what a linear model would typically be able to, though at the cost of compression noise that requires nonlinear filtering. \citet{polysemanticityandcapacity} studies polysemanticity through the lens of allocation of limited feature capacity. They show that networks tend to represent important features monosemantically, and compress together less significant features in polysemantic way. \citet{superpositionmemorizationanddoubledescent} then applies the toy models to establish the relationship between superposition, overfitting, and double descent.

\section{Limitations and Broader Impacts}
\label{sec:limitations and broader impacts}
\subsection{Limitations and Future Works}
As with other model merging methods, DRM may not be universally effective across all task combinations. Tasks with fundamentally conflicting objectives or disparate learned features may result in suboptimal merging performance, potentially leading to negative transfer on certain tasks. The effectiveness of DRM merged models in generalizing to entirely unseen tasks during merging remains unexplored. Our experiments were based on a variety of transformer models. However, experiments on other architecture types such as CNN and RNN remain to be explored. Furthermore, while we employ interference reduction techniques from TIES Merging \citep{tiesmerging} in our work, theoretical understanding of the merging interference and their mitigation is still limited, and we sought to answer more in the future.

\subsection{Impact Statement}
Here, we discuss potential societal impacts of our works. Possible positive impacts are: 
(a) Democratization of Deep Learning: model merging lowers the barrier to entry for developing deep learning models. Pretrained and finetuned open-source models can be leveraged through merging to create custom models tailored to specific needs, fostering innovation in diverse domains.
(b) Resource Efficiency: by enabling the creation of multitask models without the need for expensive retraining or joint training, DRM can significantly reduce the computational resources and energy consumption associated with training deep learning models, contributing to faster and more sustainable AI development. In contrary, our work is not tied to any one specific application, there are no obvious negative societal downside. While it might be possible that merging models with certain hidden biases inadvertently exacerbates the biases into the resulting model, this is largely dependent on the individual model being merged.

\section{Significance Analysis of the Merging Performance}
\label{sec:statistical significances}

To assess the statistical significance of the performance differences observed in Figure \ref{fig:merging different number of tasks}, we conducted paired t-tests between DRM-H and each of the compared baselines (Simple Averaging, TIES Merging, and DARE-TIES) for each subset size (number of tasks merged).

\begin{table}
    \vspace{-1.5em}
    \setlength\tabcolsep{3pt}
    \caption{P-values from paired t-tests for each subset size on ViT-B/32. The null hypothesis assumes equal performance between DRM-H and the baselines, while the alternative hypothesis assumes DRM-H performs better.}
    \centering
    \vspace{0.7em}
    \begin{tabular}{cccccccc}
        \toprule
        & \multicolumn{7}{c}{\textbf{Number of tasks}} \\
        \cmidrule(lr){2-8}
        \textbf{Method} & 2 & 3 & 4 & 5 & 6 & 7 & 8 \\
        \midrule
        Averaging & 8.41e-7 & 1.03e-11 & 2.73e-14 & 1.33e-15 & 4.40e-18 & 6.38e-17 & 2.18e-3  \\
        Ties & 9.96e-1 & 1.49e-3 & 1.35e-10 & 1.60e-12 & 6.69e-16 & 7.85e-14 & 4.61e-3 \\
        Dare-ties & 9.99e-1 & 7.04e-4 & 9.83e-11 & 1.62e-13 & 4.45e-16 & 2.74e-14 & 3.80e-3 \\
        \bottomrule
    \end{tabular}
    \label{tab:vit_sig_test_for_num_tasks}
\end{table}

\begin{table}
    \setlength\tabcolsep{3pt}
    \caption{P-values from paired t-tests for each subset size on T5-Base. The null hypothesis assumes equal performance between DRM-H and the baselines, while the alternative hypothesis assumes DRM-H performs better.}
    \centering
    \vspace{0.7em}
    \begin{tabular}{cccccc}
        \toprule
        & \multicolumn{5}{c}{\textbf{Number of tasks}} \\
        \cmidrule(lr){2-6}
        \textbf{Method} & 2 & 3 & 4 & 5 & 6 \\
        \midrule
        Averaging & 1.68e-3 & 4.20e-6 & 1.73e-10 & 2.28e-7 & 2.83e-2 \\
        Ties & 8.76e-1 & 3.68e-2 & 1.58e-2 & 6.72e-2 & 2.32e-1  \\
        Dare-ties & 9.81e-1 & 2.23e-1 & 2.01e-1 & 1.48e-1 & 3.17e-1 \\
        \bottomrule
    \end{tabular}
    \label{tab:t5_sig_test_for_num_tasks}
\end{table}

The results of the paired t-tests are summarized in Table \ref{tab:vit_sig_test_for_num_tasks} and \ref{tab:t5_sig_test_for_num_tasks}.  We report the p-values obtained for each comparison at each subset size.  P-values less than 0.05 were considered statistically significant.

In ViT-B/32, as the number of tasks increases, the performance differences between DRM-H and the baselines become statistically significant, the p-values consistently fall below 0.05 starting at 3 tasks. This suggests the superior performance of DRM-H in scaling to a larger number of tasks.

\section{Theoretical Details}

\subsection{Biased Magnitude-based Pruning of Partitioned Singular Basis}
\label{sec:Biased Magnitude-based Pruning of Partitioned Singular Basis}

\begin{prop}[Shared Norm Budget of Partitioned Orthonormal Basis Vectors]
\label{prop:appendix_norm_budget}
    Let $V \in \mathbb{R}^{kd \times r}$ be a matrix with orthonormal columns, such that $V^T V = I_r$. Let $V^T$ be partitioned horizontally into $k$ blocks corresponding to $k$ tasks, $V^T = [V_1^T \;\; V_2^T \; \dots \; V_k^T]$, where each block $V_t^T \in \mathbb{R}^{r \times d}$. Let $v_{t,i}^T$ denote the $i$-th row vector of block $V_t^T$.
    
    Then, for any given row index $i \in \{1, \dots, r\}$, the sum of the squared Euclidean norms of the corresponding row vectors across all task blocks is equal to one:
    $$
    \sum_{t=1}^{k} \|v_{t,i}^T\|_2^2 = 1 \,.
    $$
    This demonstrates that the squared norms of the partitioned row vectors for any basis direction $i$ are constrained by a shared unit budget.
\end{prop}

\begin{proof}
    We begin with the orthonormality property of the matrix $V$:
    $$
    V^T V = I_r \,.
    $$
    Substituting the partitioned form of $V^T$ and its corresponding transpose $V$:
    $$
    [V_1^T \;\; V_2^T \; \dots \; V_k^T]
    \begin{bmatrix}
        V_1 \\
        V_2 \\
        \vdots \\
        V_k
    \end{bmatrix}
    = \sum_{t=1}^{k} V_t^T V_t = I_r \,.
    $$
    This equation establishes an identity between the sum of matrix products and the identity matrix $I_r$. We are interested in the diagonal entries of this identity. The $i$-th diagonal entry is given by:
    $$
    \left( \sum_{t=1}^{k} V_t^T V_t \right)_{ii} = \sum_{t=1}^{k} (V_t^T V_t)_{ii} = (I_r)_{ii} = 1 \,.
    $$
    The term $(V_t^T V_t)_{ii}$ represents the inner product of the $i$-th row of $V_t^T$ (which is $v_{t,i}^T$) with the $i$-th column of $V_t$. Since the $i$-th column of $V_t$ is the transpose of the row vector $v_{t,i}^T$, this inner product is equivalent to the squared Euclidean norm of the row vector:
    $$
    (V_t^T V_t)_{ii} = v_{t,i}^T (v_{t,i}^T)^T = \|v_{t,i}^T\|_2^2 \,.
    $$
    Substituting this back into the sum gives the final result:
    $$
    \sum_{t=1}^{k} \|v_{t,i}^T\|_2^2 = 1 \,.
    $$
    This holds for all row indices $i=1, \dots, r$.
\end{proof}

Due to the existence of this constraint, a larger norm for one task’s vector necessitates a smaller norm for the others. The connection from this norm budget to the entry magnitude is given in the next corollary.

\begin{cor}[Larger RMS Magnitude in Vector with Larger Norm Share]
\label{cor:pruning_bias}
An unequal distribution of this norm share (i.e. $\|v_{A,i}^T\|_2^2 \neq \|v_{B,i}^T\|_2^2$ for two tasks $A$ and $B$) directly implies that the root-mean-square (RMS) magnitude of entries is higher in the vector with the larger norm share.

\end{cor}

\begin{proof}
Consider two task-specific $d$-dimensional row vectors, $v_{A,i}^T$ and $v_{B,i}^T$, for a given basis direction $i$. Assume without loss of generality that task A has a larger norm share for this direction, such that
$$
\|v_{A,i}^T\|_2 > \|v_{B,i}^T\|_2 \;.
$$
Since the norms are non-negative, squaring both sides preserves the inequality:
$$
\|v_{A,i}^T\|_2^2 > \|v_{B,i}^T\|_2^2 \;.
$$
The root-mean-square (RMS) magnitude of a vector $v$ of dimension $d$ is defined as $\text{RMS}(v) = \sqrt{\frac{1}{d}\sum_{j=1}^d v_j^2} = \frac{\|v\|_2}{\sqrt{d}}$.

Applying this definition to our two vectors gives
$$
\text{RMS}(v_{A,i}^T) = \frac{\|v_{A,i}^T\|_2}{\sqrt{d}} \quad \text{and} \quad \text{RMS}(v_{B,i}^T) = \frac{\|v_{B,i}^T\|_2}{\sqrt{d}} \;.
$$
Given our initial assumption that $\|v_{A,i}^T\|_2 > \|v_{B,i}^T\|_2$, it directly follows that:
$$
\text{RMS}(v_{A,i}^T) > \text{RMS}(v_{B,i}^T) \;.
$$

Thus, a larger norm share necessarily leads to a higher RMS magnitude.
\end{proof}

Under the reasonable assumption that the entries of $v_{A,i}^T$ and $v_{B,i}^T$ are drawn from distributions of the same shape, their scale will be determined by their respective RMS magnitudes. Consequently, the entries of the vector with the lower norm share are statistically smaller. Therefore, \uline{applying a uniform magnitude-based pruning across these un-renormalized vectors will  disproportionately remove entries from the vector with the lower norm share}.

\subsection{Bound of Difference Between Individual Task and the Concatenated Matrix}
\label{sec:Bound of Difference Between Individual Task and the Concatenated Matrix}

This section provides derivation and formal discussion of the Proposition~\ref{prop:Bounded difference of weight delta concatenation} (Bounded difference of weight delta concatenation) from Section~\ref{sec:How Large is the Difference between Individual and the Concatenated Weight Delta}. We begin by stating a theorem proven by \citet{perturbationanalysisofsingularvalueinconcatmatrices}:

\begin{lemma}[Singular Value Perturbation of a Concatenated Matrix Theorem in \citep{perturbationanalysisofsingularvalueinconcatmatrices}]
\label{lemma:perturbation_matrix}

Let \(\{A_j\}_{j=1}^k\) be a collection of matrices with \(A_j \in \mathbb{R}^{m \times n}\) for all \(j \in [1, k]\). Define the original concatenated matrix by
\[
M = [A_1 \; \ldots \; A_k] \in \mathbb{R}^{m \times kn},
\]
and denote its rank by \(r = \operatorname{rank}(M)\). Let \(\widetilde{M} = [\widetilde{A}_1 \; \ldots \; \widetilde{A}_k]\) be the perturbed version of \(M\), and define the perturbation matrix
\[
E = \widetilde{M} - M = [E_1 \; \ldots \; E_k], \quad \text{with } E_j = \widetilde{A}_j - A_j \text{ for each } j \in [1, k].
\]
Then, the following perturbation bounds for the singular values hold:

\begin{itemize}
    \item For \(i = 1, \ldots, r\) (corresponding to the nonzero singular values of \(M\)),
    \[
    \left| \sigma_i(\widetilde{M}) - \sigma_i(M) \right| \leq \frac{1}{\sigma_i(M)} \sum_{j=1}^k \left( 2\|A_j\|_2 \|E_j\|_2 + \|E_j\|_2^2 \right).
    \]

    \item For \(i = r+1, \ldots, \min(m, kn)\) (corresponding to the zero singular values of \(M\)),
    \[
    \sigma_i(\widetilde{M}) \leq \sqrt{ \sum_{j=1}^k \left( 2\|A_j\|_2 \|E_j\|_2 + \|E_j\|_2^2 \right) }.
    \]
\end{itemize}

\end{lemma}

Then, we consider the relationship between singular value of a single matrix, compared to when it is repeatedly concatenated into a horizontal block matrix:
\begin{lemma} \label{lem:eigen_of_concate_matrix}
    Suppose that $M\in \mathbb{R}^{m \times kn}$ is a horizontal concatenation of $k$ copies of $A\in\mathbb{R}^{m \times n}$; then,
    \[
    \sigma_i(M) = \sqrt{k}\sigma_i(A)
    \]
    where $\sigma_i(\cdot)$ denotes the $i$-th singular value.
\end{lemma}
\begin{proof}[Proof]
First, we consider the $MM^T$ as a block matrix multiplication:
\begin{align*}
    MM^T = [A \;\; A \;\; \cdots \;\; A] 
\begin{bmatrix}
A^T \\
A^T \\
\vdots \\
A^T
\end{bmatrix} 
= \sum_{j=1}^k AA^T  
= k AA^T.
\end{align*}
For an eigenvector $x_i$ corresponding to $\lambda_i(AA^T)$, we see that 
\begin{align*}
MM^Tx_i = kAA^T x_i = k \lambda_i(AA^T) x_i.
\end{align*}
This implies $\lambda_i(MM^T) = k \lambda_i(AA^T)$. By definition, the singular value of matrix $M$ equals the square-root of the corresponding eigenvalue of $MM^T$:
\begin{align*}
    \sigma_i(M) = \sqrt{\lambda_i(MM^T)} = \sqrt{k \lambda_i(AA^T)} = \sqrt{k} \cdot \sqrt{\lambda_i(AA^T)} = \sqrt{k} \cdot \sigma_i(A).
\end{align*}

\end{proof}

Building on the previous arguments, we now establish the proof of the main Proposition~\ref{prop:Bounded difference of weight delta concatenation} stated in the Section~\ref{sec:How Large is the Difference between Individual and the Concatenated Weight Delta}:

\begin{proof}[Proof of Proposition~\ref{prop:Bounded difference of weight delta concatenation}] From Lemma \ref{lemma:perturbation_matrix}, we obtain the perturbation bound for the singular values as follows:
    \begin{align*}
        |\sigma_i(\widetilde{M}) - \sigma_i(M)| &\leq \frac{1}{\sigma_i(M)} \sum_{j=1}^k \left( 2 \|A_{j}\|_2 \|E_{j}\|_2 + \|E_{j}\|_2^2 \right) \quad \text{for all $i \in [1, r]$}.
    \end{align*}
Then, we substitute $\sigma_i(M)$ with $\sqrt{k}\sigma_i(A)$, following Lemma~\ref{lem:eigen_of_concate_matrix}:
    \[
    |\sigma_i(\widetilde{M}) - \sqrt{k}\sigma_i( A )| \leq \frac{1}{\sqrt{k} \sigma_i(A)} \sum_{j=1}^k \left( 2 \|A\|_2 \|E_{j}\|_2 + \|E_{j}\|_2^2 \right) \quad \text{for all $i \in [1, r]$}.
    \]
Finally, we substitute back $A= \Delta W_t$, acquiring
    \[
    |\sigma_i(\widetilde{M}) - \sqrt{k}\sigma_i( \Delta W_t )| \leq \frac{1}{\sqrt{k} \sigma_i( \Delta W_t )} \sum_{j=1}^k \left( 2 \| \Delta W_t \|_2 \|E_{j}\|_2 + \|E_{j}\|_2^2 \right) \quad \text{for all $i \in [1, r]$}.
    \]
    
\end{proof}

The coefficient $\sqrt{k}$ naturally arises for our DRM approach, which involves concatenating $k$ matrices; however, this factor is generally not problematic as the number of tasks $k$ is typically small and remains relatively consistent for each merging setting. Likewise, vertically concatenated matrices can also be bounded by going through the same steps on the transposed version, producing a very similar inequality due to the invariance under transpose property of singular value and spectral norm.

\paragraph{Limitation.}
While Proposition~\ref{prop:Bounded difference of weight delta concatenation} establishes the upper bound of difference between $i$-th singular values of (1) weight delta of an individual task, and (2) jointly concatenated weight delta; however, note that $i$-th basis vectors of the two matrices doesn't necessarily correspond to the exact same basis direction.

\section{Analysis of Merging through Different Decompositions}
\label{sec: Analysis of Merging Models through Different Decompositions}
Whilst we choose to employ SVD in this work, it is natural to consider alternative decomposition techniques. This section analyzes the implication of different decompositions to our problem of model merging.

\subsection{Singular Value Decomposition}
One of the most widely used decomposition technique, SVD factorizes a matrix into a rotation, rescaling, and another rotation respectively. Namely, given a real matrix $A \in \mathbb{R}^{m \times n}$, we have the SVD:
\[
A  = U \Sigma V^T,
\]
where left singular vector matrix $U \in \mathbb{R}^{m \times m}$ is a real orthogonal matrix with its column describing the column space of $A$, right singular vector matrix $V^T \in \mathbb{R}^{n \times n}$ is a real orthogonal matrix with its row describing the row space of $A$; and singular value matrix $\Sigma \in \mathbb{R}^{m \times n}$ is a non-negative real diagonal matrix, with singular values as its diagonal entries. The number of non-zero singular values equals the rank of the matrix $A$.

SVD is desirable for our approach due to three key properties: 
\paragraph{Decoupling between basis scales and directions.} 
As demonstrated earlier, applying SVD on the concatenated weight deltas decomposes them into a matrix of shared basis ($U$), a scaling matrix ($\Sigma$), and individual basis matrices ($V^T_t$) that are aligned inside the same shared space. The decoupling of directions captured by singular vectors, and scales captured by singular values, grants us the ability to work on the feature direction and magnitude separately. Allowing us to organize our DRM algorithm as proposed in Section~\ref{sec:proposedmethod}: (1) pruning the scale-less basis vector matrices, then (2) performing sign election and disjoint averaging on the scaled basis vector matrices. See matrix depiction of our method in Figure~\ref{fig:matrix_viewpoint_diagram}.

\paragraph{Invariance of concatenation orders.} 
Different orders of weight delta concatenation: $[A \; B] $ and $ [B \; A]$ should produce the same shared space represented by $U$ (or very similar in practice). This invariance is crucial in ensuring the robustness of the decomposition process. As we treat each task as equal peers, shuffling the task concatenation order should not generate different merging outcomes.

\paragraph{Factorization into column and row basis matrices.} 
This enables us to meaningfully concatenate the weight deltas in either horizontal (DRM-H) or vertical (DRM-V) direction, analogous to decomposing into a shared column or row space respectively. We discuss the implications and distinctions between the twos thoroughly in Section~\ref{sec: connection between v and h}.

Moreover, the sorted singular value property of SVD also enables direct manipulation of matrix rank. Low-rank approximation can be obtained simply by zeroing out smaller singular value entries \citep{eckart-young}. However, this property was not utilized in DRM as we observed that dropping rank leads to decrease in performance (Section~\ref{sec:rank dropping}).

While SVD offers numerous benefits, it also has limitations. First of all, SVD is computationally more intensive compared to other matrix decompositions, and can become a bottleneck as the sizes of the weight delta matrices grow. We mitigate this by running SVD on a GPU via the cuSOLVER library. Furthermore, SVD is non-unique. Namely, multiplying both a left singular vector $u_i$ and its corresponding right singular vector $v_i$ by -1 yields another valid SVD. This non-uniqueness arises because SVD identifies a \emph{subspace} of the matrix, not the exact feature, and the sign of basis vectors can flip without effecting the space. However, this is not problematic for our use case. In DRM-H, as $U$ is always shared across weight deltas, if a single column $u_i$ is sign flipped, then row basis $v_{t, i}$ of every task will also be flipped, leading to no net effect in downstream sign election. The same holds for DRM-V.

Here is one simple illustration of the non-uniqueness of SVD. Let $A = \begin{bmatrix} 1 & 0 \\ 0 & 2 \end{bmatrix}$, a valid SVD is $ \begin{bmatrix} 1 & 0 \\ 0 & 1 \end{bmatrix} \begin{bmatrix} 1 & 0 \\ 0 & 2 \end{bmatrix} \begin{bmatrix} 1 & 0 \\ 0 & 1 \end{bmatrix}$. Another valid SVD is obtained by flipping the sign of the second singular vector: $ \begin{bmatrix} 1 & 0 \\ 0 & -1 \end{bmatrix} \begin{bmatrix} 1 & 0 \\ 0 & 2 \end{bmatrix} \begin{bmatrix} 1 & 0 \\ 0 & -1 \end{bmatrix} $.

\begin{figure}
    \centering
    \includegraphics[width=1.0\linewidth]{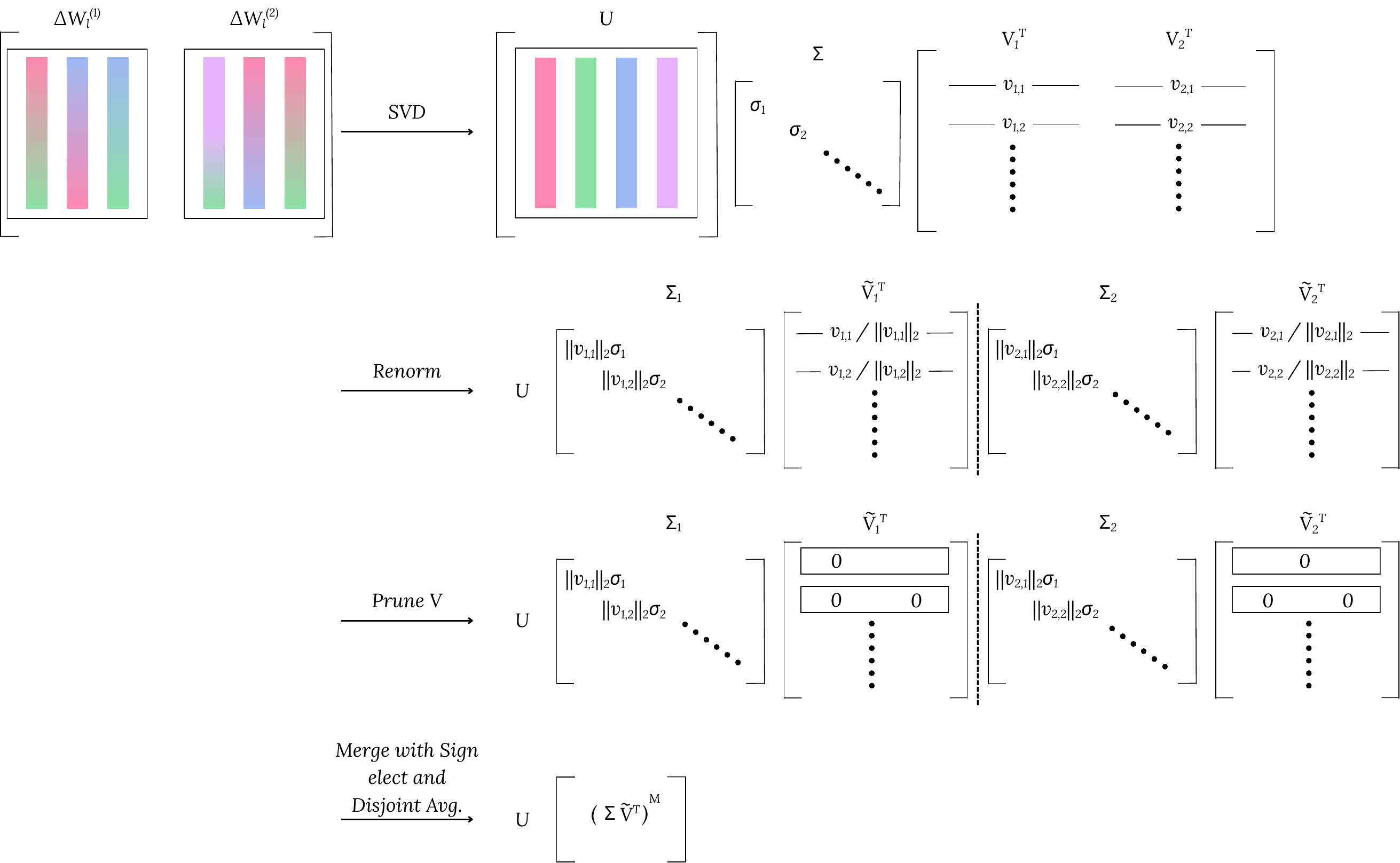}
    \caption{A matrix viewpoint of DRM-H. \textbf{(First row)} Horizontal joint decomposition. \textbf{ (Second row) } Renormalization on row basis vectors. \textbf{(Third row)} Pruning of the right singular vector matrices. \textbf{(Fourth row)} Sign election and disjoint averaging of right singular vector matrices. } 
    \label{fig:matrix_viewpoint_diagram}
\end{figure}

\subsection{QR Decomposition}
The QR decomposition factorizes a matrix into a set of column basis vectors, and their column operations. Namely, given a matrix $A \in \mathbb{R}^{m \times n}$, we have the QR decomposition:
\[ A = QR ,\]
where $Q \in \mathbb{R}^{m \times m}$ is an orthogonal matrix, and $R \in \mathbb{R}^{m \times n}$ is an upper triangular matrix.

Similar what we have done with the SVD, we can apply the QR decomposition to the horizontally concatenated weight deltas, and view the step as aligning into a shared column space described by $Q$. Notwithstanding, \emph{QR decomposition is strongly dependent to the ordering of columns}, and by extension, strongly dependent to the order of weight deltas concatenation . The first column of $Q$ will always be the normalized first column of $A$ (or negative counterpart of the same normalized vector), due to the upper triangular constraint of $R$. Additionally, if we were to mimic the step of splitting $V^T$ into $ \{ V_t^T\}_{t=1}^N$ in DRM for representing each model, splitting $R$ produces submatrices $ \{ R_t \}_{t=1}^N$ with imbalanced distribution of zero entries due to the triangular structure of $R$. These behavior could introduce bias in the subsequent interference reduction and merging steps.

\subsection{Canonical Polyadic Decomposition}
While joint decomposition through SVD effectively establishes a shared representation space for interference reduction and merging, 2-dimensional structure of matrix may not fully capture the complex interdependencies between models being merged. We briefly explore the use of Canonical Polyadic (CP) decomposition, a tensor factorization technique that can potentially model higher-order interactions across multiple models.

CP decomposition factorizes a tensor into a sum of $k$ rank-one tensors \citep{tensordecompositions}.
For our application, we can stack $p$ weight delta matrices in a new axis, forming a tensor $ \mathcal{T} \in \mathbb{R}^{p \times m \times n} $. We have the CP decomposition: 
\begin{align*}
\mathcal{T} = \sum_{r=1}^{R} \lambda_r(\mathbf{a}_r \otimes \mathbf{b}_r \otimes \mathbf{c}_r)
\end{align*}
where $R$ is the rank of the CP decomposition, $\lambda_r$ is the scaling factor of the $r$-th rank. $ \mathbf{a}_r \in \mathbb{R}^{p} $, $ \mathbf{b}_r \in \mathbb{R}^{m} $, and $ \mathbf{c}_r \in \mathbb{R}^{n} $ are the factor vectors for the $r$-th component. The symbol $ \otimes $ denotes the Kronecker product. Each rank-one tensor represents a basis component of the overall tensor, capturing a specific pattern of interaction between weight deltas.

In principle, we could adapt DRM to work with the CP decomposition. Analogously, we could partition the factor vector $\mathbf{a}_r$ into $p$ segments, each representing the contribution of a single model to the $r$-th component. Then, we multiply the segment back with other two corresponding factor vectors of the $r$-th rank, and apply interference reduction and merging techniques to these segments in the shared CP space.

However, we encounter significant obstacles with this approach, primary of which is the inherent computation burden of CP decomposition. In our experiments using TensorLy \citep{tensorly}, merging through CP decomposition took up hours on even our smallest model ViT-B/32. Determining the rank of a tensor is also NP-hard \citep{nphardtensors}, demanding tuning search to find the optimal rank hyperparameter.

\section{Elaboration of Sign Agreement Analysis}
\label{sec:elaboration of Sign Agreement Analysis}
Already investigated in Section~\ref{sec:How Does Pruning on Decompose-renormalized Space Reflect Interference in Original Parameter Space}, using sign agreement as a measure of interference. Here we compliment the previous investigation, and visualize detailed histograms comparing the distribution of sign agreements after pruning the original and the Decompose-renormalized spaces of ViT-B/32 and Llama3.1-8B, in Figure~\ref{fig:sign_agreement_stack_hist}. 

Particularly, for the histogram labels \textit{Horizontal} and \textit{Vertical}, we first project the weight deltas into the Decompose-renormalized space, apply pruning to the basis vector entries, and measure the sign agreements. Afterwards, we transform the results back to the original space and measure sign agreement, labeled \textit{Horizontal (Original Space)} and \textit{Vertical (Original Space)}. For the histogram label \textit{Original Weight}, we simply directly prune the weight deltas and report their sign agreements.

Note that that the minimum sign agreement is $50\%$, when there are an equal numbers of positive and negative entries. Notably, for the bin of $0.9-1.0$ sign agreement level, transforming back to the original space showcases a sign agreement drop of approximately 3 folds, on both models we experimented on. This supports our conclusion in Section~\ref{sec:How Does Pruning on Decompose-renormalized Space Reflect Interference in Original Parameter Space}.

\begin{figure}
    \centering
    \includegraphics[width=1.0\linewidth]{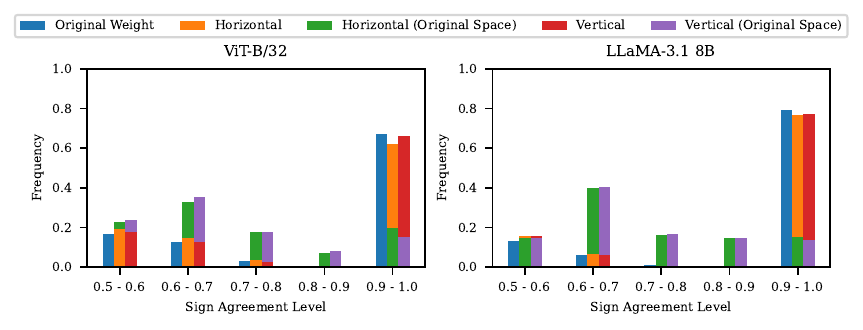}

     \caption{Histogram of sign agreement after pruning the original and the Decompose-renormalized spaces of Vit-B/32 and LLaMA-3.1 8B. For each position in the weight delta matrix, we tally the agreement in sign across different tasks. Then, the agreement across all positions are aggregated into a histogram. The range of sign agreement is between $0.5$ and $1.0$, where $0.5$ denotes a position that has an equal number of positive and negative values, while $1.0$ refers to a position where all tasks have the same sign. We visualize the sign agreement in original parameter space (\textit{blue}), the joint singular space (\textit{orange} and \textit{red}), and after projecting joint singular space back to the original space (\textit{green} and \textit{purple}).}
    
    \label{fig:sign_agreement_stack_hist}
\end{figure}

\label{sec:Characteristic of the Decompose-renormalized Space}

\section{Horizontal vs. Vertical Joint Decomposition}
\label{sec: connection between v and h}

This work introduces two variants of DRM: (1) \textit{DRM-H}, which performs joint decomposition on horizontally concatenated weight deltas, and (2) \textit{DRM-V}, which operates on vertically concatenated weight deltas. These processes can be interpreted as decomposing and coordinating weight deltas into a shared column space (DRM-H) or a shared row space (DRM-V). Empirically, we observe that DRM-H demonstrated a superior performance over DRM-V across all settings evaluated \underline{without} hyperparameter tuning, including ViT-B/32, ViT-L/14, DeBERTa-Base, T5-Base, T5-Large, and Llama3.1-8B (with a surplus as large as $12.8\%$ in DeBERTa-Base). On the other hand, when held-out validation set was available for tuning, DRM-V managed to achieve higher performance compared to DRM-H on T5-Base, T5-Large, and Llama3.1-8B ($0.7\%$, $2.3\%$, and $1.0\%$ absolute improvements, respectively). This suggests that DRM-H is generally more robust and outperforms DRM-V across a wide range of merging settings. While DRM-V may be more sensitive to hyperparameter settings, it can occasionally achieve higher performance on specific models and tasks with careful tuning.

One possible consideration for the two variants concerns the shape of the weight matrices. While many linear transformations within the self-attention mechanism of transformers utilize square matrix, certain models employ rectangular projection weights. For example, models using Grouped Query Attention \citep{gqa}, such as Llama3.1-8B \citep{llama, llama3} used in our experiments, often adopt rectangular key and value projection weights with more columns than rows. This difference in input and output dimensionality results in different effective dimension of the Decompose-renormalized space for us to work on. Concretely, for weight deltas of shape $m\times n$, DRM-H aligns weight deltas into a shared column space, and work on the partitioned row basis vectors of dimension $n$, while DRM-V aligns the deltas into a mutual row space, and work on the partitioned column basis vectors of dimension $m$. Due to \textit{frequent orthogonality} property of high dimensional space, vectors are much more likely to be orthogonal in higher dimensions. Prior work \citet{taskarithmetic} speculated that higher orthogonality between task vectors corresponds to lower interference during model merging. Given this, if $m \neq n$, then the disparate degree of orthogonality between basis vectors could be a source of the observed differences in merging outcomes between DRM-H and DRM-V.

Additionally, model architectures and training schemes seem to play a role as well. As we observe that DRM-H consistently exceeds DRM-V's performance when merging encoder-based models: ViT-B/32, ViT-L/14, DeBERTa-Base. Then, DRM-V outperforms when performing \underline{tuned} merging on encoder-decoder-based and decoder-only models: T5-Base, T5-Large, and Llama3.1-8B. We leave in-depth study of the effect and connection between DRM-H and DRM-V as future work.

\section{Additional Experiments and Results}
\subsection{Effect of Rank Truncation}
\label{sec:rank dropping}

Given that SVD employed in DRM factorizes the concatenated weight deltas into a spectrum of singular values, it is natural to consider the effect of rank truncation for the possibility of further parameter interference reduction and improved merging performance. We explore dropping small singular values (i.e. correspond to dropping less important singular vectors) to demonstrate the effect of rank truncation. After applying SVD to $\Delta W_l^{\text{stack}}$, we retain only the top $(1 - k) \cdot r$ ranks, where $r$ is the rank of $\Delta W_l^{\text{stack}}$, and $k$ is the rank drop rate; before proceeding with the merging procedure. As show in Figure \ref{fig:rank_truncation}, the performance remains relatively stable, and only begins to degrade noticeably after $k$ is greater than 0.8. However, as truncating ranks does not result in distinct performance increment, we did not employ this procedure in DRM. The absence of improvements may indicate that the lower-ranked singular vectors, while contributing less individually, still capture valuable information; or that the interference caused by these components is negligible. Future work could explore alternative methods in manipulating the less important singular vectors for potential interference reduction.

\begin{figure}
    \centering
    \includegraphics[width=1.0\linewidth]{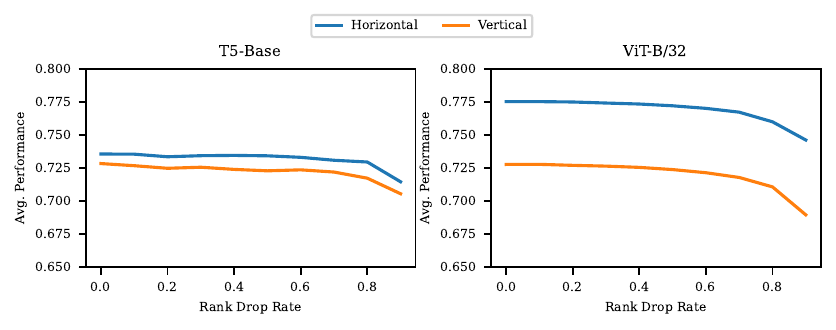}
    \caption{Performance of T5-Base and ViT-B/32 when dropping small singular values. A rank drop rate of $k$ means retaining only the top $(1 - k)$ portion of the singular values before renormalization (multiplying vector norms back into the singular values during renormalization could cause changes in singular value ordering). We see that the performance remains relatively stable, with only a slight decreasing trend. The performance only begins to degrade noticeably after $k$ is greater than 0.8.}
    \label{fig:rank_truncation}
\end{figure}

\subsection{Ablation of Interference Reduction Components}
\label{sec:ablation of interference reduction components}

To evaluate the contribution of each interference reduction technique \citep{tiesmerging} employed in DRM-H, we perform an ablation analysis by removing each component individually and evaluate the resulting performance. We remove: (1) pruning (2) sign election (3) disjoint averaging (replaced with simple averaging). Table \ref{tab:ablation of interference reduction components} shows the results of this analysis.

\begin{wraptable}{r}{0.5\textwidth}
    \centering
    \caption{Ablation of interference reduction techniques employed in DRM-H}
    \begin{tabular}{@{}lcc@{}}
        \toprule
        \textbf{Method} & \textbf{ViT-B/32} & \textbf{T5-Base} \\
        \midrule
        \textsc{DRM-H} & \textbf{77.5} & 73.6 \\
        \midrule
        \quad-- Prune & 74.9 & 72.4 \\
        \quad-- Elect Sign & 76.6 & \textbf{76.7} \\
        \quad-- Disjoint Average & 62.9 & 57.5 \\
        \bottomrule
    \end{tabular}
    \label{tab:ablation of interference reduction components}
\end{wraptable}

We recognize that disjoint averaging is the most critical component of DRM-H, removing this technique causes significant performance drops in both ViT-B/32 and T5-Base up to $16.1\%$ absolutely. Surprisingly, we also find that merging without sign election improves the merging performance by $3.1\%$ on T5-Base, this result warrants further investigation in future work.

\subsection{Justification of Joint Pruning}
In the pruning step, two possible approaches can be made: (1) joint pruning: concatenate singular vector matrices (i.e. $\widetilde{V}_t$ for DRM-H and $\widetilde{U}_t$ for DRM-V) across all tasks before sorting the entries for top-$k\%$ and then prune the rest. (2) individual pruning: prune each singular vector matrix separately. To justify our design choice, we summarize the comparison between these two approaches in Table~\ref{tab:joint_and_indiv_prune}.

\begin{table}[h]
    \caption{Comparison of merging results with joint and individual pruning.}
    \centering
    \begin{tabular}{ccccc}
        \toprule
        Method & Trim & ViT-B/32 & T5-Base & DeB-Base \\
        \midrule
        \multirow{2}{*}{DRM-H} & Individual & 77.4 & 73.6 & 66.0 \\
        & Joint & 77.5 & 73.6 & 66.9 \\
        \midrule
        \multirow{2}{*}{DRM-V} & Individual & 72.8 & 72.5 & 53.2  \\
        & Joint & 72.8 & 72.8 & 54.1 \\
        
        \bottomrule
    \end{tabular}
    \label{tab:joint_and_indiv_prune}
\end{table}

In ViT-B/32 and T5-Base, the performances between joint and individual pruning are comparable, with joint pruning performing slightly better or equal for both DRM-H and DRM-V. On the other hand, DeBerta-Base is more sensitive, with joint pruning achieving 0.9\% absolute improvement compared to individual pruning when merging with either DRM-H or DRM-V. Motivated by the improvement, our method employs joint pruning.

\subsection{Hyperparameter Sensitivity}
\label{sec:hyperparameter sensitivity}
To assess the robustness of DRM, we conduct a sensitivity analysis to evaluate how the performance of merging techniques varies with different hyperparameter settings, when merging ViT-B/32, T5-Base, and Llama3.1-8B. We focus on two hyperparameters: (1) Pruning retention rate $k$. (2) Merging Coefficient $\lambda$.

We report the result in Figure \ref{fig:hyperparameter sensitivity}. The analysis suggests that DRM is generally more robust to hyperparameter variations within a reasonable range, compared to TIES Merging and DARE-TIES, especially when merging ViT-B/32 and T5-Base, but careful tuning is still crucial for enhancing the merging results. This information can guide future selection of appropriate hyperparameter values for DRM. The \emph{accuracy reported here were measured on validation set}, reflecting the hyperparameter tuning search.

\begin{figure}
    \centering
    \includegraphics[width=1.0\linewidth]{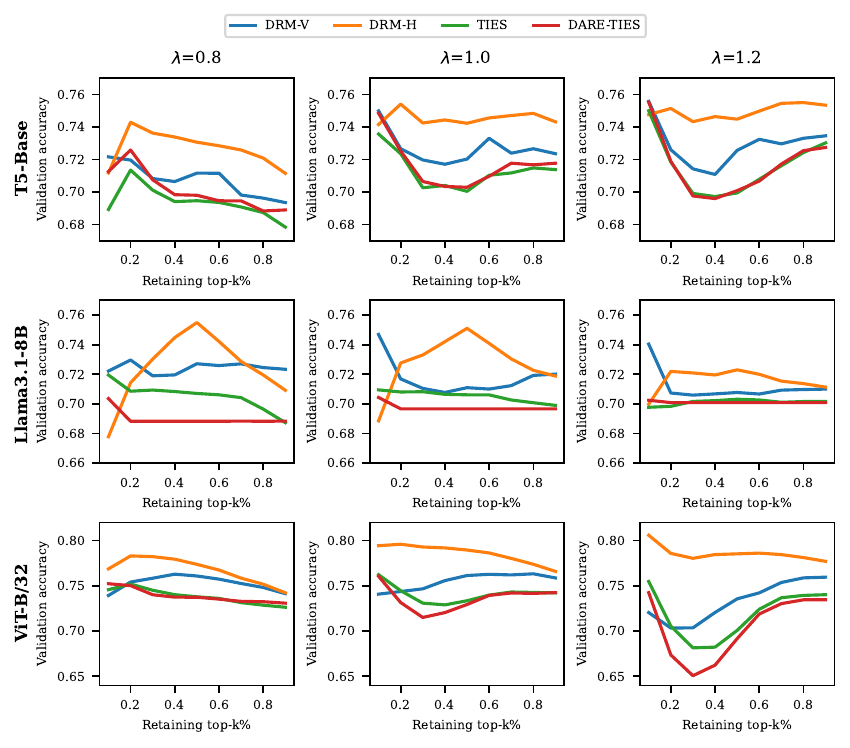}
    \caption{Hyperparameter sensitivity when merging ViT-B/32, T5-Base, and Llama3.1-8B. We fix merging coefficient $\lambda$ to one of $\{0.8, 1.0, 1.2\}$, and try varying the parameter retention rate $k$ (inverse of pruning rate). We see that the performance of both DRM-H and DRM-V is generally more robust to hyperparameter variations within a reasonable range, compared to TIES Merging and DARE-TIES. Validation set accuracy is reported here, to reflect the hyperparameter tuning search.}
    \label{fig:hyperparameter sensitivity}
\end{figure}

\subsection{Hyperparameter Tuning}
\label{sec:hyperparameter tuning}

To ensure fair and optimal comparisons, we report in Section \ref{sec:experiments} both merging performance on default hyperparameters, and performance using hyperparameters tuned via grid search on held-out validation sets described in Section \ref{sec:training details}. This section details the hyperparameter search spaces.

For each merging technique, we tuned the following hyperparameters:
\begin{itemize}
    \item \textbf{Task Arithmetic:} The merging coefficient $\lambda$.
    \item \textbf{TIES Merging:} Prune rate, and merging coefficient $\lambda$.
    \item \textbf{DARE-TIES:} Prune rate, and merging coefficient $\lambda$.
    \item \textbf{DRM-H \& DRM-V:} Prune rate, and merging coefficient $\lambda$.
\end{itemize}
Following \citet{taskarithmetic, tangenttaskarithmetic}, every method uses a single merging coefficient $\lambda$ value shared across every task.

For Task Arithmetic, we explored a merging coefficient search space $\lambda \in \{ 0.1, 0.2, 0.3, \dots ,1.0\}$. For TIES Merging, DARE-TIES, DRM-H and DRM-V, we specify the search space for merging coefficient $\lambda \in \{ 0.8, 0.9, 1.0, \dots , 1.5\}$; and prune rate $k \in \{ 0.1, 0.2, 0.3, \dots , 1.0\}$, the range is selected to be exhaustive regardless of whether the pruning is defined by prune rate or retention rate. Notice that the Task Arithmetic employs $\lambda$ as a coefficient for weighted sum, while other methods utilize $\lambda$ on top of averaging. Therefore, their search spaces cover different scales.

\subsection{Per-task Performance Breakdown}
\label{sec:performance breakdown}

To provide a more detailed analysis of the performance of our proposed DRM method, this section presents a breakdown of the results for each individual task in our evaluation suite. Complementing the aggregate performance presented in Section \ref{sec:experiments}. All the reported results are classification accuracy.

Specifically, refer to
\begin{itemize}
    \item Table~\ref{tab:vit-base performance breakdown} for per-task performance of ViT-B/32,
    \item Table~\ref{tab:vit-large performance breakdown} for per-task performance of ViT-L/14,
    \item Table~\ref{tab:t5-base performance breakdown} for per-task performance of T5-Base,
    \item Table~\ref{tab:t5-large performance breakdown} for per-task performance of T5-Large,
    \item Table~\ref{tab:deberta-base performance breakdown} for per-task performance of DeBERTaV3-Base,
    \item Table~\ref{tab:llama performance breakdown} for per-task performance of Llama3.1-8B.
\end{itemize}

\section{Training Details}
\label{sec:training details}
\subsection{Vision Models}
Vision models employed in our study are CLIP ViT-B/32 and CLIP ViT-L/14, each with a linear classification head attached for each task. Both the models and classifiers were provided by \citet{taskarithmetic} \footnote{\url{https://github.com/mlfoundations/task_vectors}}. The checkpoints follow the OpenCLIP implementation \citep{openclipsoftware}, where the self-attention's query, key, and value projection weights are stacked into a single weight matrix. We found no discernible performance difference on applying our method to the stacked QKV versus each projection weight separately. Hence, we apply our methods on the stacked matrix to keep the implementation straightforward.

\subsubsection{Datasets}
For vision tasks, we used the following datasets: MNIST \citep{mnist} under Gnu General Public License. EuroSAT \citep{eurosat} under MIT License. Cars \citep{stanfordcars} and GTSRB \citep{gtsrb} under Creative Commons License. We could not find the license information of DTD \citep{dtd}, RESISC45 \citep{resisc45}, SUN397 \citep{sun397} and SVHN \citep{svhn}. We attempted to reproduce the exact splits for EuroSAT and RESISC45 from \citet{taskarithmetic} but were unable to do so, and instead followed the dataset sizes reported in that work.
 
\subsection{Language Models}
We chose the two variants of encoder-decoder models T5-Base\footnote{\url{https://huggingface.co/google-t5/t5-base}} and T5-Large\footnote{\url{https://huggingface.co/google-t5/t5-large}}, as well as another encoder-only model DeBERTaV3-Base\footnote{\url{https://huggingface.co/microsoft/deberta-v3-base}}.

\subsubsection{Finetuning hyperparameters}
The model were finetuned through Hugging Face's Transformers library (\emph{-ForSequenceClassification} class family), using the AdamW \citep{adamw} optimizer with a learning rate of 1e-4, a batch size of 1024 samples, and a maximum of 50,000 optimization steps with early stopping enabled. Following \citet{tiesmerging}, we did not use learning rate scheduling and weight decay in our experiments, as they could introduce additional confounding factors. Bfloat16 precision was enabled for better training efficiency. A single A100 80GB GPU was used for training.

\subsubsection{Datasets}
We employed a diverse set of datasets, covering a range of natural language understanding tasks. Datasets were accessed via Hugging Face's Datasets. For datasets where validation sets were not provided, we split $10\%$ of the training set with random seed 0 to create a validation set. The following datasets were used:

\textbf{PAWS}\citep{paws}\footnote{\url{https://huggingface.co/datasets/google-research-datasets/paws}}: we used the original \textit{labeled} split.

\textbf{QASC}\citep{qasc}\footnote{\url{https://huggingface.co/datasets/allenai/qasc}}: with CC-BY 4.0 license. We used a total of 8,134 training samples, 926 validation samples, and 920 test samples.

\textbf{QuaRTz}\citep{quartz}\footnote{\url{https://huggingface.co/datasets/allenai/quartz}}: with CC-BY 4.0 license. We used the original split.

\textbf{StoryCloze}\citep{storycloze}\footnote{\url{https://huggingface.co/datasets/LSDSem/story_cloze}} comes with \textit{2016} and \textit{2018} splits. We used validation set of \textit{2018} split for training, and leave the \textit{2016} validation and test sets as they are.

\textbf{WikiQA}\citep{wikiqa}\footnote{\url{https://huggingface.co/datasets/microsoft/wiki_qa}}: we used the original split.

\textbf{WinoGrande}\citep{winogrande}\footnote{\url{https://huggingface.co/datasets/allenai/winogrande}}: we used the \textit{winogrande\_l} training split, and partition 10\% of it for validation. The original validation set was used for test instead. The final split consisted of 9,210 training samples, 1,024 validation samples, and 1,267 test samples.

Moreover, we attempted with finetuning on WSC \citep{wsc}, using the \textit{wsc.fixed} subset of the SuperGLUE benchmark \citep{superglue}. We tried splitting the training set to create a held-out validation set. However, we failed to train WSC models to convergence, so they were not included in our experiments.

\subsection{Large Language Models}
For LLMs, we finetuned Llama3.1-8B\footnote{\url{https://huggingface.co/meta-llama/Llama-3.1-8B}} using the \emph{LlamaForSequenceClassification} class from Hugging Face's Transformers library. Due the the high cost of training LLMs, we opted for LoRA \citep{lora} through the PEFT library, instead of full finetuning.

\subsubsection{Finetuning hyperparameters}
The models were finetuned with a low learning rate of 3e-5, and a 1e-5 weight decay (we observed that using no weight decay at all hurts the finetuned performance considerably) via the AdamW optimizer. The training was done in a distributed fashion with an effective batch size of 256 samples, limited to a maximum of 2,000 total steps, with early stopping enabled, 5\% learning rate warm up steps, and 1024 tokens context length. We applied LoRA adapters with rank 16 to every linear layer of the self-attention modules, and set the alpha to 32. Bfloat16 datatype was used together with Flash-Attention2 \citep{flashattention2} in order to minimize memory usage and speed up the training cycles. Additional performance enhancements included the use of Liger kernel \citep{ligerkernel} and PyTorch's \textit{compile} utility. We conducted the training using 4 A100 GPUs, each with 80 GBs of VRAM.

\subsubsection{Datasets}

We evaluated our approach on the following datasets from the GLUE benchmark \citep{glue}: QNLI\footnote{\url{https://huggingface.co/datasets/nyu-mll/glue/viewer/qnli}}, MNLI\footnote{\url{https://huggingface.co/datasets/nyu-mll/glue/viewer/mnli}}, RTE\footnote{\url{https://huggingface.co/datasets/nyu-mll/glue/viewer/rte}}, SST-2\footnote{\url{https://huggingface.co/datasets/nyu-mll/glue/viewer/sst2}}, and CoLA\footnote{\url{https://huggingface.co/datasets/nyu-mll/glue/viewer/cola}}, chosen to cover a range of natural language understanding tasks. Since every dataset here lacks publicly available test labels, we used the original validation sets for testing, and created new validation splits by reserving $10\%$ of the training data, except for SST-2 and CoLA, which use only 5\%, using stratified sampling with a fixed random seed of 0. The dataset statistics are summarized in Table \ref{tab:llm datasets}.

\begin{table}[h]
    \centering
    \caption{LLM Dataset Statistics}
    \label{tab:llm datasets}
    \begin{tabular}{lcccc}
        \toprule
        Dataset & Task & Training Size & Validation Size & Test Size \\
        \midrule
        QNLI \citep{glue} & Question-Answering & 94,268 & 10,475 & 5,463 \\
        MNLI \citep{mnli} & Natural Language Inference & 353,431 & 39,271 & 9,815 \\
        RTE \citep{glue} & Textual Entailment & 2,241 & 249 & 277 \\
        SST-2 \citep{sst2} & Sentiment Analysis & 63,981 & 3,368 & 872 \\
        CoLA \citep{cola} & Linguistic Acceptability & 7,695 & 856 & 1,043\\
        \bottomrule
    \end{tabular}
\end{table}

\begin{table}[p]
    \centering
    \caption{Break down of ViT-B/32 merging performance.}
    \small
    \setlength{\tabcolsep}{3pt}
    \begin{tabular}{l c c c c c c c c c c}
        \toprule
        && \multicolumn{8}{c}{\textbf{Datasets}} \\
        \cmidrule(lr){3-11}
        \textbf{Methods} & \textbf{Validation} & Cars & DTD & EuroSAT & GTSRB & MNIST & RESISC45 & SUN397 & SVHN & \textbf{Average} \\
        \midrule
        Finetuned & & 77.7 & 79.6 & 99.7 & 98.7 & 99.7 & 98.5 & 75.0 & 97.5 & \textbf{90.8}\\
        \midrule
        Simple Averaging & $\times$ & 63.4 & 50.4 & 72.8 & 52.8 & 87.5 & 72.5 & 65.0 & 64.1 & \textbf{66.1} \\
        Task Arithmetic & $\times$ & 40.9 & 42.8 & 65.3 & 66.0 & 98.1 & 54.3 & 36.5 & 80.6 & \textbf{60.6} \\
        TIES Merging & $\times$ & 58.8 & 54.3 & 78.6 & 72.1 & 98.3 & 72.4 & 59.4 & 86.2 & \textbf{72.5} \\
        Dare-TIES & $\times$ & 55.8 & 52.8 & 78.3 & 71.3 & 98.4 & 69.6 & 56.4 & 86.3 & \textbf{71.1} \\
        \rowcolor[HTML]{E3ECFA}
        DRM-H & $\times$ &64.7 & 59.7 & 89.3 & 75.9 & 98.7 & 81.8 & 65.8 & 84.4 & \textbf{77.5} \\
        \rowcolor[HTML]{F0E3FA}
        DRM-V & $\times$ &56.1 & 53.6 & 80.0 & 74.5 & 98.5 & 73.9 & 59.0 & 86.4 & \textbf{72.8} \\
        \midrule
        Task Arithmetic & \checkmark & 62.0 & 52.1 & 78.3 & 64.9 & 94.0 & 73.4 & 63.5 & 74.4 & \textbf{70.3} \\
        TIES Merging & \checkmark & 61.8 & 57.1 & 75.1 & 73.1 & 98.4 & 77.6 & 63.8 & 87.3 & \textbf{74.3} \\
        Dare-TIES & \checkmark & 63.9 & 56.6 & 76.0 & 72.2 & 97.8 & 78.3 & 65.6 & 85.1 & \textbf{74.4} \\
        \rowcolor[HTML]{E3ECFA}
        DRM-H & \checkmark & 65.4 & 60.6 & 87.4 & 76.3 & 99.0 & 83.1 & 67.1 & 87.8 & \textbf{78.3} \\
        \rowcolor[HTML]{F0E3FA}
        DRM-V & \checkmark & 62.8 & 55.0 & 83.2 & 74.5 & 98.3 & 75.1 & 62.6 & 85.8 & \textbf{74.6} \\
        \bottomrule
    \end{tabular}
    \label{tab:vit-base performance breakdown}
\end{table}

\begin{table}[p]
    \centering
    \caption{Break down of ViT-L/14 merging performance.}
    \small
    \setlength{\tabcolsep}{3pt}
    \begin{tabular}{l c c c c c c c c c c}
        \toprule
        & \multicolumn{8}{c}{\textbf{Datasets}} \\
        \cmidrule(lr){3-11}
        \textbf{Methods} & \textbf{Validation} & Cars & DTD & EuroSAT & GTSRB & MNIST & RESISC45 & SUN397 & SVHN & \textbf{Average} \\
        \midrule
        Finetuned & & 92.3 & 84.4 & 99.9  & 99.2 & 99.7 & 99.0 & 81.9 & 98.1 & \textbf{94.3} \\
        \midrule
        Simple Averaging & $\times$ & 81.6 & 62.7 & 91.1 & 70.6 & 97.0 & 83.7 & 71.9 & 78.2 & \textbf{79.6} \\
        Task Arithmetic & $\times$ & 79.2 & 64.3 & 90.5 & 86.6 & 99.1 & 85.6 & 72.3 & 89.2 & \textbf{83.4} \\
        TIES Merging & $\times$ & 84.9 & 68.8 & 95.3 & 83.3 & 99.0 & 91.2 & 76.1 & 90.3 & \textbf{86.1} \\
        Dare-TIES & $\times$ & 84.3 & 68.2 & 95.0 & 83.6 & 99.1 & 90.8 & 75.8 & 90.6 & \textbf{85.9} \\
        \rowcolor[HTML]{E3ECFA}
        DRM-H & $\times$ & 86.7 & 71.9 & 96.5 & 89.3 & 99.2 & 93.4 & 76.3 & 90.7 & \textbf{88.0} \\
        \rowcolor[HTML]{F0E3FA}
        DRM-V & $\times$ & 83.6 & 68.3 & 95.5 & 86.5 & 99.2 & 90.6 & 75.3 & 91.3 & \textbf{86.3} \\
        \midrule
        Task Arithmetic & \checkmark & 82.1 & 65.8 & 93.1 & 86.8 & 98.9 & 88.1 & 73.8 & 87.9 & \textbf{84.6} \\
        TIES Merging & \checkmark & 85.8 & 70.6 & 95.1 & 84.1 & 99.2 & 91.9 & 76.7 & 91.5 & \textbf{86.9} \\
        Dare-TIES & \checkmark & 85.7 & 70.2 & 95.2 & 84.5 & 99.1 & 91.9 & 76.6 & 91.3 & \textbf{86.8} \\
        \rowcolor[HTML]{E3ECFA}
        DRM-H & \checkmark & 87.8 & 75.2 & 96.6 & 91.9 & 99.2 & 94.7 & 77.1 & 91.6 & \textbf{89.3} \\
        \rowcolor[HTML]{F0E3FA}
        DRM-V & \checkmark & 84.0 & 67.1 & 95.9 & 87.5 & 99.2 & 90.4 & 75.1 & 91.9 & \textbf{86.4} \\
        \bottomrule
    \end{tabular}
    \setlength{\tabcolsep}{6pt}
    \label{tab:vit-large performance breakdown}
\end{table}

\begin{table}[p]
    \centering
    \caption{Break down of T5-Base merging performance.}
    \small
    \setlength{\tabcolsep}{5pt}
    \begin{tabular}{l c c c c c c c c }
        \toprule
        && \multicolumn{7}{c}{\textbf{Datasets}} \\
        \cmidrule(lr){3-9}
        \textbf{Methods} & \textbf{Validation} & PAWS & QASC & Quartz & StoryCloze & WikiQA & WinoGrande & \textbf{Average} \\
        \midrule
        Finetuned & & 94.9 & 99.5 & 75.0 & 86.1 & 96.0 & 59.0 & \textbf{85.1} \\
        \midrule
        Simple Averaging & $\times$ & 58.9 & 54.8 & 54.6 & 65.2 &95.3 & 51.2 & \textbf{63.3} \\
        Task Arithmetic & $\times$ & 84.1 & 90.6 & 59.7 & 72.6 & 95.1 & 56.6 & \textbf{76.4} \\
        TIES Merging & $\times$ & 87.1 & 63.9 & 63.4 & 63.7 & 95.3 & 57.2 & \textbf{71.8} \\
        Dare-TIES & $\times$ & 87.2 & 65.7 & 63.7 & 63.9 & 95.1 & 58.3 & \textbf{72.3} \\
        \rowcolor[HTML]{E3ECFA}
        DRM-H & $\times$ & 87.8 & 74.1 & 59.2 & 67.8 & 95.2 & 57.3 & \textbf{73.6} \\
        \rowcolor[HTML]{F0E3FA}
        DRM-V & $\times$ & 87.5 & 75.1 & 60.8 & 61.9 & 95.0 & 56.7 & \textbf{72.8} \\
        \midrule
        Task Arithmetic & \checkmark & 84.1 & 90.6 & 59.7 & 72.6 & 95.1 & 56.6 & \textbf{76.4} \\
        TIES Merging & \checkmark & 87.6 & 71.5 & 63.3 & 67.0 & 95.2 & 57.8 & \textbf{73.7} \\
        Dare-TIES & \checkmark & 86.8 & 77.2 & 63.3 & 68.0 & 95.2 & 57.2 & \textbf{74.6} \\
        \rowcolor[HTML]{E3ECFA}
        DRM-H & \checkmark & 87.1 & 78.2 & 60.2 & 67.0 & 95.3 & 58.6 & \textbf{74.4} \\
        \rowcolor[HTML]{F0E3FA}
        DRM-V & \checkmark & 86.8 & 88.7 & 61.5 & 62.5 & 95.0 & 56.3 & \textbf{75.1} \\
        \bottomrule
    \end{tabular}
    \setlength{\tabcolsep}{6pt}
    \label{tab:t5-base performance breakdown}
\end{table}

\begin{table}[p]
    \centering
    \caption{Break down of T5-Large merging performance.}
    \small
    \setlength{\tabcolsep}{5pt}
    \begin{tabular}{l c c c c c c c c }
        \toprule
        && \multicolumn{7}{c}{\textbf{Datasets}} \\
        \cmidrule(lr){3-9}
        \textbf{Methods} & \textbf{Validation} & PAWS & QASC & Quartz & StoryCloze & WikiQA & WinoGrande & \textbf{Average} \\
        \midrule
        Finetuned & & 94.5 & 99.5 & 86.2 & 90.4 & 96.0 & 69.0 & \textbf{89.3} \\
        \midrule
        Simple Averaging & $\times$ & 79.6 & 73.1 & 66.7 & 69.6 & 95.3 & 55.0 & \textbf{73.2} \\
        Task Arithmetic & $\times$ & 82.1 & 93.3 & 71.4 & 77.5 & 93.0 & 63.3 & \textbf{80.1} \\
        TIES Merging & $\times$ & 83.5 & 63.2 & 65.1 & 64.7 & 91.6 & 65.9 & \textbf{72.3} \\
        Dare-TIES & $\times$ & 83.9 & 62.8 & 65.4 & 65.0 & 91.4 & 65.4 & \textbf{72.3} \\
        \rowcolor[HTML]{E3ECFA}
        DRM-H & $\times$ & 86.0 & 73.1 & 66.6 & 63.1 & 93.5 & 64.6 & \textbf{74.5} \\
        \rowcolor[HTML]{F0E3FA}
        DRM-V & $\times$ & 86.4 & 68.7 & 62.8 & 63.0 & 93.5 & 65.3 & \textbf{73.3} \\
        \midrule
        Task Arithmetic & \checkmark & 82.1 & 93.3 & 71.4 & 77.5 & 93.0 & 63.3 & \textbf{80.1} \\
        TIES Merging & \checkmark & 82.6 & 91.4 & 71.2 & 73.7 & 92.5 & 65.8 & \textbf{79.5} \\
        Dare-TIES & \checkmark & 83.1 & 91.2 & 70.7 & 73.4 & 92.9 & 65.7 & \textbf{79.5} \\
        \rowcolor[HTML]{E3ECFA}
        DRM-H & \checkmark & 85.2 & 72.7 & 67.7 & 67.0 & 93.5 & 64.5 & \textbf{75.1} \\
        \rowcolor[HTML]{F0E3FA}
        DRM-V & \checkmark & 85.2 & 86.8 & 65.9 & 68.6 & 93.1 & 64.8 & \textbf{77.4} \\
        \bottomrule
    \end{tabular}
    \setlength{\tabcolsep}{6pt}
    \label{tab:t5-large performance breakdown}
\end{table}

\begin{table}[p]
    \small
    \setlength{\tabcolsep}{5pt}
    \centering
    \caption{Break down of DeBERTa-Base merging performance.}
    \begin{tabular}{l c c c c c c c c }
        \toprule
        && \multicolumn{7}{c}{\textbf{Datasets}} \\
        \cmidrule(lr){3-9}
        \textbf{Methods} & \textbf{Validation} & PAWS & QASC & Quartz & StoryCloze & WikiQA & WinoGrande & \textbf{Average} \\
        \midrule
        Finetuned & & 94.5 & 99.4 & 90.6 & 52.2 & 94.7 & 68.9 & \textbf{83.4} \\
        \midrule
        Simple Averaging & $\times$ & 53.4 & 18.5 & 75.3 & 49.7 & 94.7 & 54.2 & \textbf{57.6} \\
        Task Arithmetic & $\times$ & 45.1 & 25.5 & 73.6 & 49.2 & 83.2 & 55.5 & \textbf{55.4} \\
        TIES Merging & $\times$ & 50.2 & 47.7 & 62.3 & 51.1 & 26.9 & 54.0 & \textbf{48.7} \\
        Dare-TIES & $\times$ & 49.1 & 46.3 & 60.6 & 51.2 & 23.4 & 53.5 & \textbf{47.3} \\
        \rowcolor[HTML]{E3ECFA}
        DRM-H & $\times$ & 84.0 & 58.6 & 77.4 & 51.4 & 72.4 & 57.3 & \textbf{66.9} \\
        \rowcolor[HTML]{F0E3FA}
        DRM-V & $\times$ & 61.7 & 44.0 & 55.1 & 53.5 & 56.0 & 54.3 & \textbf{54.1} \\
        \midrule
        Task Arithmetic & \checkmark & 50.3 & 26.8 & 77.2 & 49.7 & 88.7 & 55.8 & \textbf{58.1} \\
        TIES Merging & \checkmark & 56.6 & 36.2 & 71.7 & 49.5 & 71.5 & 57.7 & \textbf{57.2} \\
        Dare-TIES & \checkmark & 57.7 & 34.9 & 71.6 & 50.3 & 73.9 & 58.5 & \textbf{57.8} \\
        \rowcolor[HTML]{E3ECFA}
        DRM-H & \checkmark & 84.3 & 55.1 & 76.2 & 51.4 & 84.9 & 55.4 & \textbf{67.9} \\
        \rowcolor[HTML]{F0E3FA}
        DRM-V & \checkmark & 76.6 & 50.9 & 55.9 & 51.5 & 81.7 & 50.3 & \textbf{61.1} \\
        \bottomrule
    \end{tabular}
    \label{tab:deberta-base performance breakdown}
\end{table}

\begin{table}[p]
    \centering
    \caption{Break down of Llama3.1-8B merging performance.}
    \begin{tabular}{l c c c c c c c}
        \toprule
        && \multicolumn{6}{c}{\textbf{Datasets}} \\
        \cmidrule(lr){3-8}
        \textbf{Methods} & \textbf{Validation} & COLA & SST2 & MNLI & QNLI & RTE & \textbf{Average} \\
        \midrule
        Finetuned & & 83.4 & 96.6 & 91.4 & 95.2 & 85.9 & \textbf{90.5} \\
        \midrule
        Simple Averaging & $\times$ & 72.8 & 61.9 & 47.5 & 61.0 & 64.3 & \textbf{61.5} \\
        Task Arithmetic & $\times$ & 43.8 & 49.8 & 33.1 & 50.7 & 52.4 & \textbf{45.9} \\
        TIES Merging & $\times$ & 70.7 & 51.4 & 76.5 & 53.5 & 83.4 & \textbf{67.1} \\
        Dare-TIES & $\times$ & 70.1 & 51.2 & 79.1 & 51.9 & 85.9 & \textbf{67.6} \\
        \rowcolor[HTML]{E3ECFA}
        DRM-H & $\times$ & 72.1 & 55.7 & 71.5 & 63.4 & 84.8 & \textbf{69.5} \\
        \rowcolor[HTML]{F0E3FA}
        DRM-V & $\times$ & 70.1 & 51.2 & 85.2 & 53.2 & 85.9 & \textbf{69.1} \\
        \midrule
        Task Arithmetic & \checkmark & 69.1 & 50.9 & 80.8 & 50.5 & 80.5 & \textbf{66.4} \\
        TIES Merging & \checkmark & 72.7 & 51.8 & 81.0 & 56.2 & 85.6 & \textbf{69.4} \\
        Dare-TIES & \checkmark & 69.6 & 50.9 & 86.3 & 50.6 & 86.6 & \textbf{68.8} \\
        \rowcolor[HTML]{E3ECFA}
        DRM-H & \checkmark & 73.2 & 63.1 & 74.7 & 70.7 & 84.8 & \textbf{73.3} \\
        \rowcolor[HTML]{F0E3FA}
        DRM-V & \checkmark & 73.9 & 68.5 & 74.3 & 75.9 & 79.1 & \textbf{74.3} \\
        \bottomrule
    \end{tabular}
    \label{tab:llama performance breakdown}
\end{table}

\end{document}